%% file: main.tex
\documentclass{article}

\usepackage{iclr2026_conference,times}
\usepackage{needspace}

\usepackage{amsmath,amsfonts,bm}
\usepackage{hyperref}
\usepackage{amsthm}
\usepackage{url}
\usepackage{graphicx}
\usepackage{subcaption}
\usepackage{cleveref}
\usepackage{amsmath}

\newif\ifshowTodos
\showTodostrue 

\title{Dual Perspectives on Non-Contrastive\\ Self-Supervised Learning}

\author{%
Jean Ponce\\
Ecole normale sup\'erieure/PSL \\
New York University\\
\texttt{jean.ponce@ens.fr} \\
\And
Basile Terver\\
Meta FAIR \\
INRIA Paris \\
\texttt{basileterv@meta.com}
\And
Martial Hebert\\
Carnegie-Mellon University\\
\texttt{martial.hebert@cs.cmu.edu}\\
\And
Michael Arbel\\
Univ. Grenoble Alpes, Inria\\
CNRS, Grenoble INP, LJK\\
\texttt{michael.arbel@inria.fr}}

\newtheorem{theorem}{Theorem}[section]
\newtheorem{proposition}[theorem]{Proposition}
\newtheorem{lemma}[theorem]{Lemma}
\newtheorem{corollary}[theorem]{Corollary}
\newtheorem{definition}[theorem]{Definition}

\iclrfinalcopy 

\begin{document}

\def\Q{{\mathcal Q}}
\def\q{{\mathfrak q}}
\def\RR{\mathbb{R}}
\def\CC{\mathbb{C}}
\def\EE{\mathbb{E}}
\def\VV{\mathbb{V}}
\def\KK{\mathbb{K}}
\def\NN{\mathbb{N}}
\def\PP{\mathbb{P}}
\def\AA{\mathbb{A}}
\def\LL{\mathbb{L}}
\def\SS{\mathbb{S}}
\def\TT{\mathbb{T}}
\def\barr{\bar{\mathbb{R}}}
\def\mat#1{{\mathcal{#1}}}
\def\vect#1{\bm{#1}}
\def\PPi{\mbox{\boldmath$\Pi$}}
\def\squig{\rightsquigarrow}
\def\eqdef{\buildrel \rm def \over =}
\def\mypar#1{\noindent{\bf #1}}
\def\comment#1{{}}
\def\qmatrix#1{\left[\begin{matrix}#1\end{matrix}\right]}
\def\st#1{{\tt #1}}

\newcommand{\ma}[1]{{\color{blue}{#1}}}

\maketitle

\begin{abstract}
The {\em stop gradient} and {\em exponential moving average} iterative procedures are commonly used in non-contrastive approaches to self-supervised learning to avoid representation collapse, with excellent performance in downstream applications in practice. This presentation investigates these procedures from the dual viewpoints of optimization and dynamical systems. We show that, in general, although they {\em do not} optimize the original objective, or {\em any} other smooth function, they {\em do} avoid collapse. Following~\citet{Tian21}, but without any of the extra assumptions used in their proofs, we then show using a dynamical system perspective that, in the linear case, minimizing the original objective function without the use of a stop gradient or exponential moving average {\em always} leads to collapse. Conversely, we characterize explicitly the equilibria of the dynamical systems associated with these two procedures in this linear setting as algebraic varieties in their parameter space, and show that they are, in general, {\em asymptotically stable}. Our theoretical findings are illustrated by empirical experiments with real and synthetic data.
\end{abstract}

\section{Introduction} {\em Self-supervised learning} (or {\em SSL})
is an approach to representation learning that exploits the internal consistency of training data {\em without} requiring expensive annotations.
It has proven to be an effective alternative to traditional supervised technology, for applications in natural language processing~\citep{Mikolov13,Vaswani17}, image
analysis~\citep{SimCLR,BYOL20,DINO,SimSiam21,CLIP,Barlow_Twins,bardes_vicreg} and video understanding~\citep{Bardes24}. 
Early SSL approaches,  e.g.,~\citep{Mikolov13,SimCLR,Moco20,CLIP}, were {\em contrastive}:  models are learned from training pairs that can be either {\em negative}, when one data point is representative of the target population and the other one is not, or {\em positive}, when both data points are representative of the target. The training consists in pushing negative pairs apart while pulling positive ones together.
{\em Non-contrastive} approaches to SSL  have emerged as a powerful alternative often outperforming contrastive ones empirically~\citep{BYOL20, SimSiam21, Bardes24, I-JEPA, bardes_vicreg, DINO}.
These techniques compute representations of different {\em views} of the same data and learn to predict one from another, thus avoiding the need for mining  negative data. They are, however, susceptible to {\em
  representational collapse} where a constant embedding is learned
~\citep{PathAMI}.

Preventing representation collapse in non-contrastive approaches has thus become a key focus in SSL, leading to two principal strategies: {\em feature decorrelation} and {\em enforcing asymmetry between the two views}.
The first strategy addresses representational collapse by explicitly enforcing decorrelation among the learned features.
For instance, \citet{bardes_vicreg} introduce a regularizer $\Omega$  designed to avoid collapse by keeping the variance of the codes of the two views of samples above a fixed threshold while encouraging the codes associated with the same sample to be similar.
More recently, \citet{sansone2025collapseproof}  proposed an auxiliary classification task with randomly assigned labels, providing theoretical guarantees against collapse.
Despite their conceptual simplicity, feature-decorrelation methods are often empirically outperformed by techniques that introduce  asymmetries between the views during training to avoid collapse.
Specifically, these rely on a teacher/student architecture, where the student computes a {\em source} view as the composition of encoder and predictor networks and aims to predict a {\em target} view obtained using the teacher network. The latter is either a frozen copy (through a {\em stop gradient} operation or {\em SG}) or a delayed version (through an {\em exponential moving average} or {\em EMA}) of the student encoder~\citep{DINOV2, Bardes24, I-JEPA, BYOL20, SimSiam21}. SG and EMA have shown strong empirical performance and remain standard components of state-of-the-art SSL models~\citep{I-JEPA, DINOV2}.

\mypar{Problem statement.} Despite the empirical success of SG and EMA, there is no obvious link between
these methods and the optimization of a well-defined objective
function.
This motivates gaining a theoretical
understanding of their behavior, including:
(a) Do SG and EMA  solve an optimization
problem and, if they do, which one?
(b) Do they converge and, when and if they do, are they guaranteed to avoid collapse?
(c) Seen as dynamical systems, are their stationary points, if any,
  stable, so there is no risk of drifting  from them
  to some trivial solution?
These are the  problems we address in the rest of this presentation,
from dual perspectives: an {\em optimization perspective} for (a) and (b), and a {\em dynamical system} one for (c), following the work of~\citep{Littwin24,Tian21,Wang21} in the {\em linear} case.

\mypar{Main contributions.}
\begin{itemize}
\item[(1)] We prove with Proposition~\ref{prop:first} that, in
  general,\footnote{
 Whenever we state
    that some property holds {\em in general},
    this means that, although they may not hold for {\em certain} data satisfying specific equations, they do hold, in practice, for {\em all} generic
data, in the standard mathematical sense, following the common notion of genericity in dynamical systems and algebraic geometry, e.g.,  \cite{hirsch2013differential}.} neither the SG algorithm nor
its EMA counterpart minimizes the  objective they are
  derived from and that, if they converge, they both avoid collapse (Figure~\ref{fig:landscape}).
\item[(2)] In the case where the loss is the squared Euclidean
  distance, we then prove with Proposition~\ref{prop:noptim} the conjecture given
  in~\citep{BYOL20} that the SG and EMA algorithms do not optimize {\em
    any} well defined function.
\item[(3)] We confirm (1) empirically  (Section~\ref{sec:realexp}) on an action
  classification task from video data, further finding that the SG and EMA algorithms do not appear to converge, although their downstream performance increases momentarily in training.
 \end{itemize}
 Following~\citet{Littwin24,Tian21,Wang21}, we then switch to a dynamical system
 perspective in the case where the encoder and predictor are both
 linear operators.
 \begin{itemize}
 \item[(4)] We characterize in Proposition~\ref{prop:Eqchar} and Corollary~\ref{corollary1} the equilibria of the dynamical systems associated with the SG and EMA
algorithms as a finite set of algebraic varieties.
 \item[(5)] We prove with Proposition~\ref{prop:dynamics} that these
   equilibria are, in general, {\em asymptotically
     stable}~\citep{Arnoldiff92}. In particular, when started close to
   them, the two procedures are guaranteed to converge there and
   stay there.
\item[(6)] We run some simulations in the simplified setting where the
  input space is scalar ($m=1$) and show that the two algorithms
  converge in general in these experiments, although possibly to trivial minima.
\end{itemize}

All proofs are relegated to the appendix for conciseness.
In the linear case, these proofs leverage the equations'  structure~\citep{matcook12} to avoid cumbersome tensor manipulations.
On a minor note, this notably allows us to rederive, for completeness, several results about the
dynamics of the SG and EMA algorithms (Lemmas~\ref{lemma:discdyn}
to~\ref{lemma:th1}) already known from~\citep{Tian21}, but without
assumptions from their original proofs, such as that the two views of the data be drawn from the same distribution conditioned on the data, or that the eigenvalues of certain PSD matrices be bounded away from zero, whose validity is difficult to guarantee in practice.

\subsection{Related work}

Several works have explored why SSL methods learn effective
representations, with feature-decorrelation methods being specifically
investigated
in~\citep{balestriero2022,whiteningloss2022,ziyin2023whatshapeslossSSL,jing2022understanding}. \citet{balestriero2022}
attempt to unify VICReg~\citep{bardes_vicreg}, SimCLR~\citep{SimCLR}
and Barlow Twins~\citep{Barlow_Twins} in a single framework using
global and local spectral embedding methods. \citet{whiteningloss2022}
investigate how whitening-based losses avoid collapse.
\citet{ziyin2023whatshapeslossSSL} develop an analytically tractable
theory of SSL loss landscapes for both contrastive and
feature-decorrelation methods, analyzing factors that affect SSL
robustness to data imbalance.

Experimental studies of asymmetry-based methods~\citep{Tao2021UniGrad,zhang2022howSimSiam,liu2022bridging} have
investigated their properties.
Notably,
UniGrad~\citep{Tao2021UniGrad} systematically compares contrastive and
non-contrastive SSL approaches, concluding that a momentum encoder is
key to improved performance.  \citet{zhang2022howSimSiam} use gradient
analysis of l2-normalized representations to study why SimSiam avoids
collapse.
\citet{liu2022bridging} experimentally show that
asymmetry-based methods like BYOL and SimSiam implicitly enforce
feature decorrelation, linking them to methods such as Barlow Twins
and VICReg.

Theoretical investigations of asymmetry-based methods have been proposed, notably by \citet{Tian21}, showing that, under
isotropic data and specific optimization trajectory assumptions, the
predictor and stop-gradient are essential to prevent collapse in
linear BYOL and SimSiam.
This study has inspired
algorithms where the predictor is an explicit function of the encoder,
with proven collapse-avoiding
properties~\citep{Wang21,jing2022understanding,halvagal2023implicit,self-predictivelearning2023}.
\citet{prediction_head_ncssl} investigate the encoder’s role in
avoiding collapse using a two-layer neural network, albeit under a
simplified data-generating process. Finally, \citet{Littwin24} study
the effect of depth on learned representations in deep linear models,
showing that Joint-Embedding Predictive Architectures (or JEPA) models prioritize ``influential'' features (features which are most informative in prediction), with
derivations relying on diagonal covariance matrices and orthogonal
initialization.

\section{Problem setting}
\begin{figure}[t]
  \vspace{-2em}
  \centerline{%
    \includegraphics[width=0.6\columnwidth]{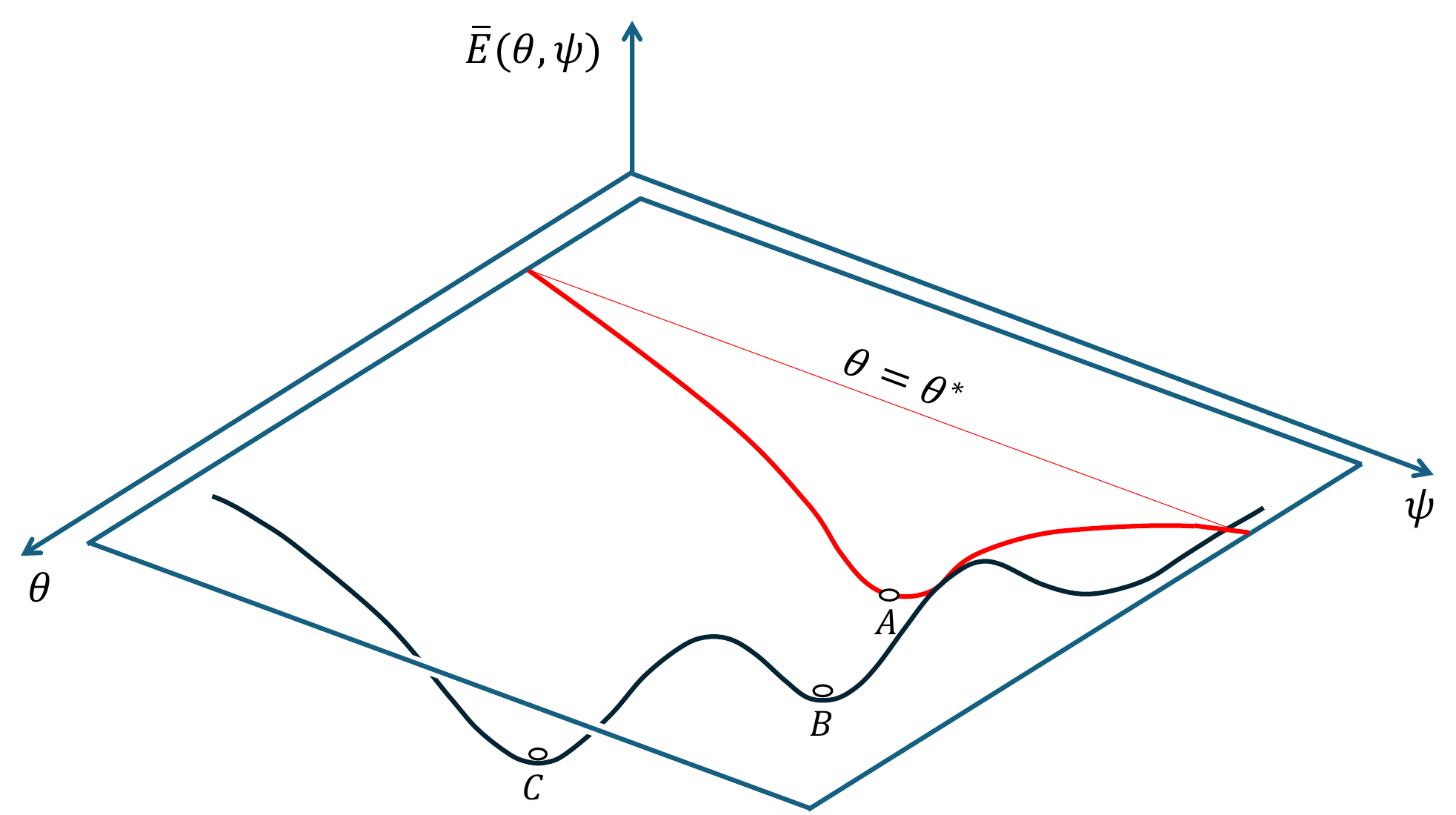}}
  \vspace{-2mm}
  \caption{\small
  A (toy) illustration of the optimization landscape for the objective function $\bar{E}(\theta,\psi)$ of Eq.~\eqref{eq:mainobj}. Here $C$ is the global minimum of $\bar{E}(\theta,\psi)$ (shown as negative instead of zero for readability)
associated with a collapse of the training process; $B$ is a nontrivial local minimum one may reach using an appropriate regularization to avoid collapse; and $A$ is a limit point of the stop gradient (SG) training procedure
associated with parameters $\theta^*$ and $\psi^*$ at convergence. In general, it is not a minimum of $\bar{E}$ and thus does not correspond to a collapse of the training process, but it is a minimum with respect to $\psi$ of $\bar{E}(\theta^*,\psi)$. See text for details.}
\vspace{-1em}
\label{fig:landscape}
\end{figure}
Given a parametric encoder $f_\theta:\RR^m\rightarrow\RR^n$ and a
parametric predictor $g_\psi:\RR^n\rightarrow\RR^n$ with parameters
$\theta$ in $\RR^p$ and $\psi$ in $\RR^q$, it is possible~\citep{PathAMI,bardes_vicreg} to learn
$\theta$ and $\psi$ from data embedded in $\RR^m$ without outside
supervision by minimizing with respect to these parameters the
objective function
\begin{equation}
  \bar{E}(\theta,\psi)=\EE_{x,y} E(\theta,\psi,x,y),
\,\,\text{where}\,\,
  E(\theta,\psi,x,y)=l[g_\psi\circ f_\theta(x),f_\theta(y)]+
  \Omega(\theta,\psi).
  \label{eq:mainobj}
\end{equation}
Here, $x$ and $y$ are different {\em views} of some data point (e.g.,
different crops of the same image), $\EE_{x,y} E$ is the mean of the
function $E$ over the (unknown) distribution of the views conditioned
on the corresponding data, approximated in
practice by a mean over a finite number of data
samples, $l:\RR^n\times \RR^n\rightarrow\RR^+$ is some loss and
$\Omega:\RR^p\times\RR^q\rightarrow\RR^+$ is some regularizer.  The
corresponding architecture is that of a Siamese
network~\citep{Siamese94} whose branches correspond to the
encoders of the two views compared, with shared parameters, while the predictor
sits on top of the first branch.
In this setting, $\bar{E}(\theta,\psi)$ can be minimized with
respect to these parameters by using, for example, stochastic gradient
descent. With proper learning rates, the training procedure will
converge to some critical point of $\bar{E}$ where both gradients are
zero (Figure~\ref{fig:landscape}). A difficulty, however, is how to
prevent it from {\em collapsing} by converging to the degenerate zero
global minimum corresponding to $f_\theta$ being a constant and
$g_\psi$ being the identity, or $f_\theta$ being zero and $g_\psi$
being a function such that $g_{\psi}(0)=0$.
To address this difficulty, {\em BYOL}~\citep{BYOL20} and {\em
  SimSiam}~\citep{SimSiam21} propose to use {\em
  exponential moving average} and {\em stop gradient} training
procedures, as defined in the rest of this section, as alternative minimization procedures. To
properly define these procedures, let us introduce an objective
function with an additional argument $\xi$ in $\RR^p$:
\begin{equation}
  \bar{F}(\theta,\psi,\xi)=\EE_{x,y} F(\theta,\psi,\xi,x,y),
\,\,\text{where}\,\,
  F(\theta,\psi,\xi,x,y)=l[g_\psi\circ
  f_\theta(x),f_\xi(y)]+
  \Omega(\theta,\psi),
\end{equation}
and consider instead the
 {\em exponential moving average} ({\em EMA})
procedure~\citep{BYOL20}. \\

\hrule
  \noindent{\bf EMA algorithm:} Initialize $\theta_0$, $\psi_0$ and
  $\xi_0$ to some values and $t$ to $1$, then repeat until convergence
  or you run out of patience:
\begin{itemize}
\item[(a)] $\theta_t\leftarrow \theta_{t-1}-\mu_t\nabla_\theta
  \bar{F}(\theta_{t-1},\psi_{t-1},\xi_{t-1})$;
\item[(b)]
  $\psi_t\leftarrow \psi_{t-1}-\nu_t\nabla_\psi
  \bar{F}(\theta_{t-1},\psi_{t-1},\xi_{t-1})$;
\item[(c)]  $\xi_t\leftarrow \alpha_t \xi_{t-1}+(1-\alpha_t)\theta_t$;
 \item[(d)] $t\leftarrow t+1$.
\end{itemize}
\noindent
\hrule

This is the procedure used to train BYOL in~\citep{BYOL20} and V-JEPA
in~\citep{Bardes24}\footnote{BYOL~\citep{BYOL20}
   assumes that what we call an encoder is  a
  proper encoder followed by a projection operator and that the loss
  acts on normalized versions of its inputs.  Both BYOL and SimSiam~\citep{SimSiam21} make the loss symmetric by having each view predict the other. This is subsumed by
  the framework presented here and does not change the conclusions of our
analysis}. The corresponding architecture is no longer a true
Siamese network because the encoders in its two branches have different
parameters,  so $f_\theta$ and $f_\xi$ are sometimes
respectively called the {\em online} (or {\em student}) and {\em target}
(or {\em teacher}) networks.
An alternative is to consider  the  {\em stop
  gradient}  procedure~\citep{SimSiam21} used  to
train SimSiam, which is  a true Siamese architecture with
identical encoders in its two branches but uses
$\nabla_\theta\bar{F}$ as a proxy for $\nabla_\theta \bar{E}$ when
updating $\theta$. \\

\noindent
\begin{minipage}{\textwidth}
\hrule
\vspace{0.5\baselineskip}
{\bf SG algorithm:}  Initialize $\theta_0$ and $\psi_0$ to
  some values and $t$ to $1$, then repeat until convergence
  or you run out of patience:
\begin{itemize}
\item[(a)] $\theta_t\leftarrow \theta_{t-1}-\mu_t\nabla_\theta
  \bar{F}(\theta_{t-1},\psi_{t-1},\theta_{t-1})$;
\item[(b)]
  $\psi_t\leftarrow \psi_{t-1}-\nu_t\nabla_\psi
  \bar{F}(\theta_{t-1},\psi_{t-1},\theta_{t-1})$;
 \item[(c)] $t\leftarrow t+1$.
\end{itemize}
\hrule
\end{minipage}

Note that $\bar{E}(\theta,\psi)=\bar{F}(\theta,\psi,\theta)$ for any
values of $\theta$ and $\psi$, but it is a priori unclear whether any
limit point of the SG procedure is related to the critical
points of $\bar{E}$.
Note also that the EMA algorithm reduces to the SG one when $\alpha_t=0$ and $\xi_0=\theta_0$. Indeed, \citet{SimSiam21} refer to SimSiam as ``BYOL {\em without} the momentum encoder'', but in practice
$\alpha_t$ is taken very close to $1$, e.g., $\alpha_t$ varies between $0.996$ and $1$ in BYOL~\citep{BYOL20}.

\section{An optimization perspective} {\em We  assume from now on
  for simplicity that $\alpha_t$ is  a constant with $\alpha_t\neq 1$ or converges to such a constant as $t$ tends to infinity.}
Under this assumption, if and when the EMA and SG procedures converge, we
have by continuity $\theta=\xi$ at a limit point because of step (c),
with the same gradient values in $\theta$ and $\psi$ as in SG. The dynamics may be different, and lead to
different limit points, but these will obey the {\em same} zero
gradient conditions at the equilibria of the
corresponding dynamical systems.

\comment{We will thus focus most of our
discussion in the rest of this presentation on the SG algorithm, the conclusions reached about its equilibria being also valid for its EMA counterpart.}

\subsection{EMA and SG objectives\label{sec:obj}}
The {\em motivation} for the EMA algorithm in BYOL \citep[Sec. 3.1,
Eq. (2)]{BYOL20} is to minimize over $\theta$, $\psi$ and $\xi$ the
mean over pairs of views $x$ and $y$ of the data the
squared Euclidean distance between the {\em prediction}
$g_\psi\circ f_\theta (x)$ obtained from the representation of $x$ and
the {\em target} encoding $f_\xi(y)$ of $y$, where $\xi$ is the
 moving average defined recursively
by step (c) of  EMA.  Note that, although
$\xi_t$ is  well defined at each $t$, an explicit
definition of $\xi$ is missing, except in the limit case if and when
 EMA converges. It is thus {\em
  a priori} unclear whether its iterations minimize a well-defined function.
Similarly, the motivation for the SG training procedure in
~\citep[Sec. 3, Eq.~(1)]{SimSiam21} is to minimize the mean
squared Euclidean distance between the prediction
$g_\psi\circ f_\theta (x)$ obtained from $x$ and the target $f_\theta(y)$, using  $\nabla_\theta \bar{F}$ as a
{\em proxy} for the true gradient $\nabla_\theta\bar{E}$ of the
corresponding objective function. It is therefore again {\em a
  priori} unclear whether it  minimizes a well-defined function.

\subsection{SG and EMA algorithms vs optimization problems}
It is important for fairness to clearly state that none of the papers
we are aware of that use the SG or EMA procedure claims to actually
minimize such functions in practice. In fact, the BYOL authors
 write~\citep{BYOL20}: ``More generally, we
hypothesize that there is no loss $L_{\theta,\xi}$ such that BYOL's
dynamics is a gradient descent on $L$ jointly over $\theta,\xi$.'' One
of our main contributions is actually to prove this conjecture (see
Proposition~\ref{prop:noptim} below).  Let us first state a simple
but important result (see the appendix for its proof).

\begin{proposition}
  The SG and EMA algorithms do not, in general, minimize the
  original objective function $\bar{E}$ of
  Eq.~(\ref{eq:mainobj}). If and when they converge, the corresponding
  solution is, in general, not a degenerate one corresponding to a zero global
  minimum of that function.
  \label{prop:first}
\end{proposition}

A harder question to answer is whether these algorithms optimize {\em any}  objective function.
{\em Let us assume from now on} for simplicity that $l$ is the (half) squared Euclidean
distance and $\Omega(\theta,\psi)=\lambda(\Vert\theta\Vert^2+\Vert\psi\Vert^2)/2$.
With this choice, we have $\nabla_u l(u,v)=u-v$, $\nabla_\theta\Omega=\theta$,
  $\nabla_\psi\Omega=\psi$ and
\begin{equation}
  \left\{
\begin{array}{l}
           \nabla_\theta F (\theta,\psi,\theta,x,y) =J_\theta u(\theta,\psi,x)^T[u(\theta,\psi,x)-v(\theta,y)]+\lambda\theta, \\
\nabla_\psi F (\theta,\psi,\theta,x,y)=J_\psi
  u(\theta,\psi,x)^T[u(\theta,\psi,x)-v(\theta,y)]+\lambda\psi,
\end{array}\right.
\label{eq:PQ}
\end{equation}
where $u(\theta,\psi,x)= g_{\psi}\circ f_{\theta}(x)$ and $v(\theta,y) = f_{\xi}(y)$.
We wish to understand whether the vector fields
$\EE_{x,y}[\nabla_{\theta}F(\theta,\psi,x,y)]$ and
$\EE_{x,y}[\nabla_{\psi}F(\theta,\psi,x,y)]$ are the
gradients of some well defined
scalar function which the SG
algorithm (and  its EMA counterpart) would presumably
minimize. The answer is negative:
\begin{proposition}
Under the (mild) assumption that $f_{\theta}$ and $g_{\psi}$ are smooth neural networks whose last layer is linear and that they are not identically $0$, i.e. there exists $\theta_0, \psi_0, x_0$ and $z_0$ so that $f_{\theta_0}(x_0)\neq 0$ and $g_{\psi_0}(z_0)\neq 0$, then the vector fields $\EE_{x,y}[\nabla_{\theta}F(\theta,\psi,x,y)]$ and
$\EE_{x,y}[\nabla_{\psi}F(\theta,\psi,x,y)]$ are not, in general, the gradient fields of any smooth function.
  \label{prop:noptim}
\end{proposition}

\begin{proof}[Proof sketch]
The full proof, given in the appendix, follows from Eq.~(\ref{eq:PQ}) and
Schwarz's integrability theorem, according to which, a {\em necessary}
  condition for these vector fields to be the gradient field of a
  smooth scalar function is that their second-order cross derivatives
  be the transposes of each other. By expressing these cross derivatives we find that their difference (taking into account the transposition operation) does not vanish in a generic sense.
  To establish genericity, we show that there exists arbitrarily small perturbations to the data distribution for which the corresponding cross derivative difference cannot be identically zero.
\end{proof}

\subsection{Experiments with real data\label{sec:realexp}}
We investigate in this section
three fundamental questions about the SG and EMA algorithms: (1)
Although they do not, in theory, minimize the original objective
$\bar{E}(\theta,\psi)$, do they minimize it in practice? (2) Do they
converge, and in particular do $\theta_t-\theta_{t-1}$ and
$\psi_t-\psi_{t-1}$ both tend toward zero as the number of training
steps increases? (3) Does the classification accuracy increase with
training time? We address these three questions in a realistic
setting with experiments on a video classification task using the
code of~\cite{Bardes24}, on the
Kinetics710 and SSv2 benchmarks~\citep{SSv2,K700-2020}.

\mypar{Experimental setting.}
To provide a realistic setting, we address the
challenging problem of self-supervised learning for action
classification in videos, a task for which
VICReg~\citep{bardes_vicreg} and  regularization-based methods
such as SimCLR~\citep{SimCLR} and Barlow Twins~\citep{Barlow_Twins}
have not proven successful {\em so far}. Indeed, the state-of-the-art
V-JEPA model of~\citet{Bardes24} uses the EMA algorithm for
training, and we use its code, kindly provided by its authors, in our
experiments, with an $\ell_1$ loss for $l$, a ViT-S encoder and a
ViT-T predictor~\citep{VITs}.  We use the unions of the Kinetics710
and SSv2 training datasets~\citep{SSv2,K700-2020} for learning the
representation.  We also learn in a supervised manner an attentive
pooling classifier as in~\citep{Bardes24}, with
the train splits of each dataset. We report the top-1 accuracy on the classification task for
each dataset on its validation splits.

We slightly modify the training setup compared to \citep{Bardes24} to make it simpler and faster.  We run the SG and EMA algorithms on
both datasets in our experiments, with 1000 training epochs (300 000 iterations), instead of 300 epochs in V-JEPA. We also
restrict the train set to Kinetics710 and SSv2, discarding
Howto100M~\cite{HowTo100M}. We use a ViT-Base encoder instead of a
ViT-L encoder and reduce the batch size to 1024. Instead of the cosine
annealing with warmup schedulers, we fix the learning rate to $0.0001$
and the weight decay  to $0.1$. Following~\citep{Bardes24},
the variable $\alpha_t$ increases from $0.998$ to $0.9996$ across
iterations. As expected, with these simplifications, the
performance of the trained models on the downstream K400 and on SSv2
video classification tasks (Figure~\ref{fig:min}, bottom) is lower
than the accuracies of 72.9 and 67.4 reported in \citep{Bardes24} for
the V-JEPA ViT-L model trained on K710 and SSv2 for a total of 900K
samples seen during training. Yet we believe that this simple setting
is sufficient to run realistic experiments and give preliminary yet
meaningful answers to the questions addressed in this section.

\begin{figure}
   \centering
    \begin{subfigure}[c]{0.325\textwidth}
\includegraphics[width=\linewidth]{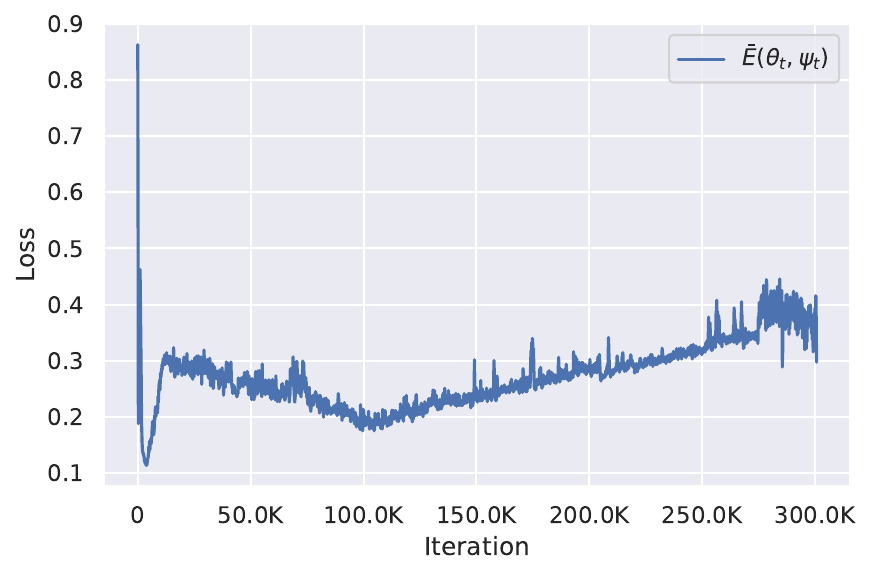}
    \end{subfigure}
        \begin{subfigure}[c]{0.325\textwidth}
\includegraphics[width=\linewidth]{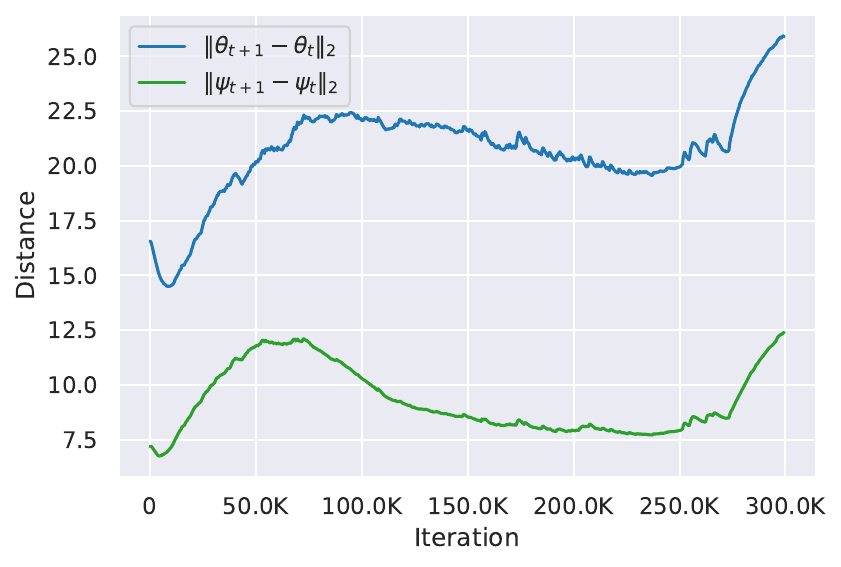}
    \end{subfigure}
        \begin{subfigure}[c]{0.325\textwidth}
    \includegraphics[width=\linewidth]{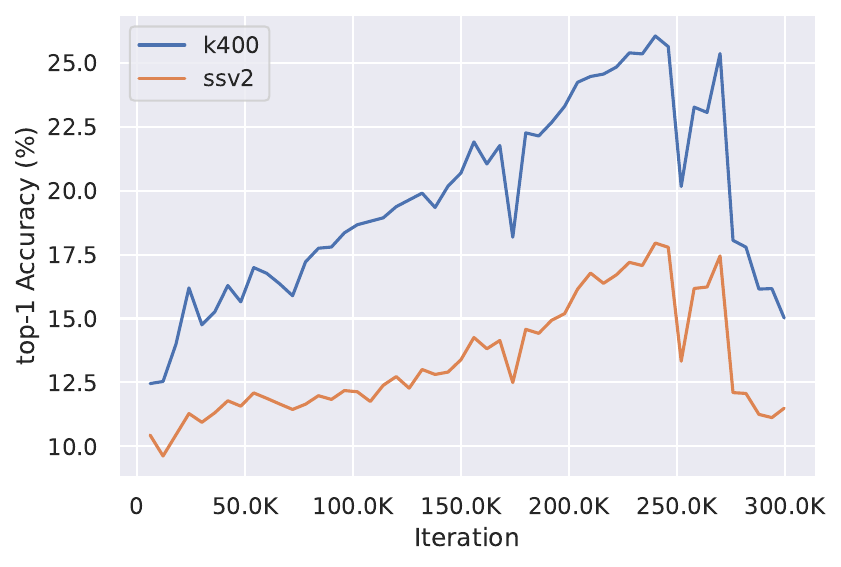}
    \end{subfigure}
    \begin{subfigure}[c]{0.325\textwidth}
    \includegraphics[width=\linewidth]{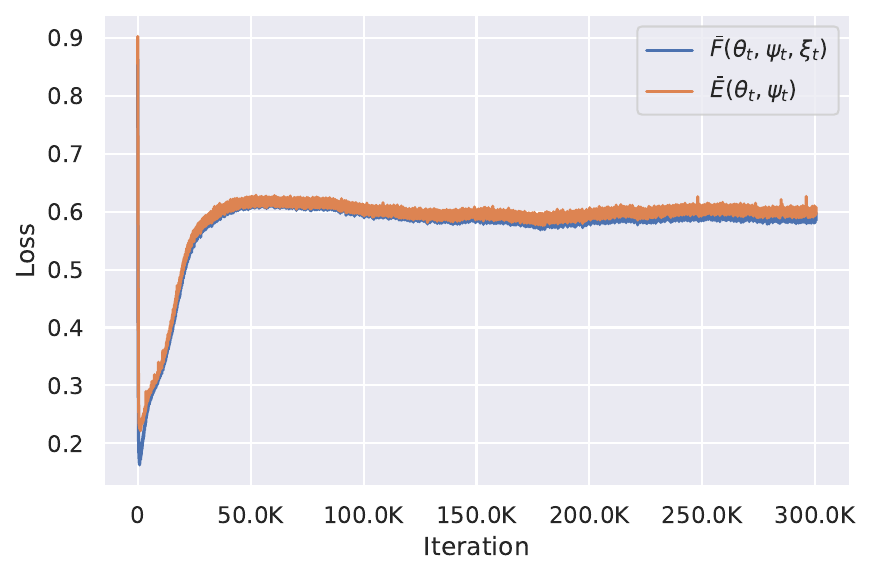}
    \end{subfigure}
    \begin{subfigure}[c]{0.325\textwidth}
    \includegraphics[width=\linewidth]{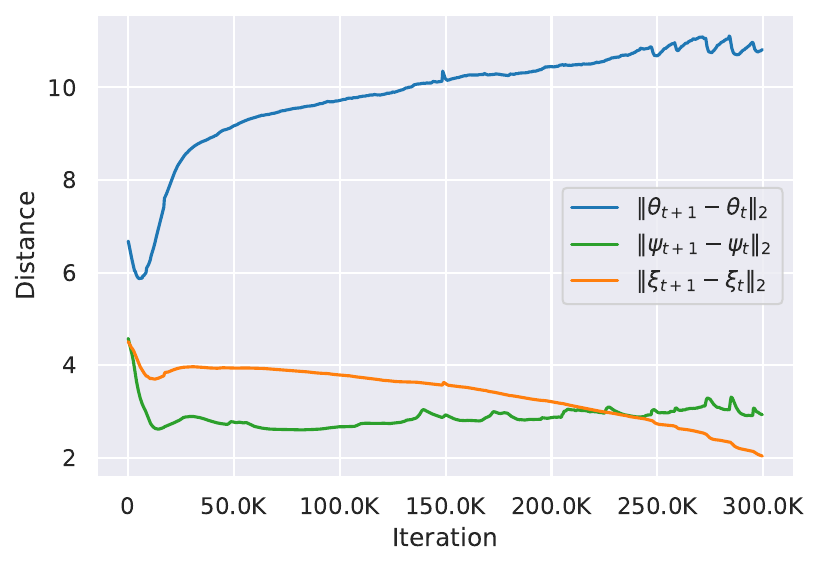}
    \end{subfigure}
    \begin{subfigure}[c]{0.325\textwidth}
    \includegraphics[width=\linewidth]{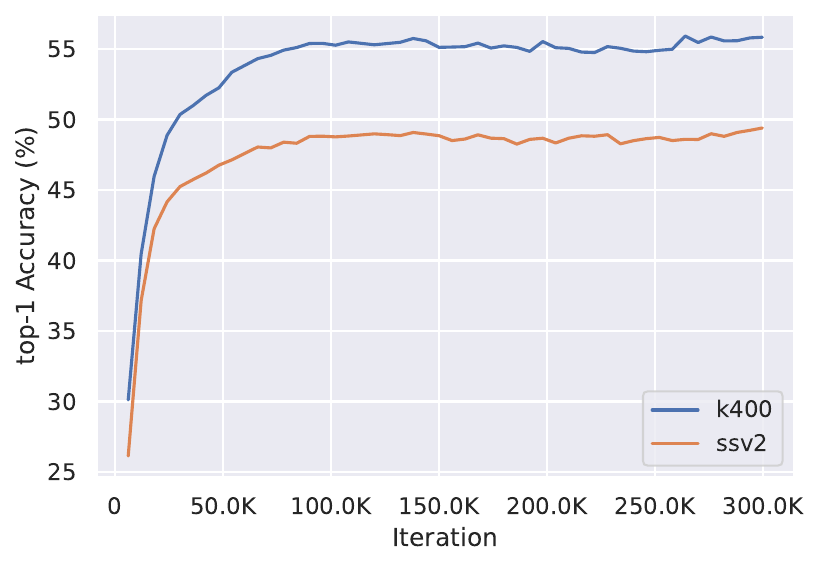}
    \end{subfigure}
    \vspace{-3mm}
    \caption{\small Evolution over 300,000 iterations of (left) the
      $\bar{E}$ objective (plus $\bar{F}$ for EMA) without the squared
      Euclidean norm regularizer, (middle) the norms of the parameter
      increments $\theta_t-\theta_{t-1}$, $\psi_t-\psi_{t-1}$ and
      $\xi_t-\xi_{t-1}$, and (right) the classification accuracy on
      the K400 and SSv2 benchmarks for SG (top) and EMA (bottom). All
      curves here are smoothed using a moving average, and the squared
      Euclidean norm regularizer is ignored to emphasize the main
      loss.}
  \vspace{-2em}
\label{fig:min}
\end{figure}

\mypar{Results and conclusions.}
Figure~\ref{fig:min} (top, left) shows the evolution of
$\bar{E}(\theta_t,\psi_t)$ for the SG procedure as a function
of the number of iterations. As expected, the algorithm does not
minimize the function, its smallest value being reached
early in the iterations. Figure~\ref{fig:min} (bottom, left) shows the
evolution of $\bar{F}(\theta_t,\psi_t,\xi_t)$ and
$\bar{E}(\theta_t,\psi_t)$ for the EMA procedure, and similar
conclusions can be reached.
As shown in Figure~\ref{fig:min} (top, middle),  $\theta$ and
$\psi$  {\em do not} converge in our experiments either
since the norms of $\theta_t-\theta_{t-1}$ and $\psi_t-\psi_{t-1}$ do
not tend to zero. As shown in Figure~\ref{fig:min} (bottom, middle), a similar
conclusion can be reached for $\theta$, $\psi$ and $\xi$
 and the EMA algorithm.
Finally, classification accuracy {\em does} increase initially  for both the SG and EMA algorithms before decreasing late
in the process for SG, and reaching a plateau for EMA
(Figure~\ref{fig:min}, right), confirming the commonly accepted fact
that they learn something, although it is still unclear exactly what
they learn. In this {\em particular} setting, EMA
 gives better classification results than SG. {\em General} conclusions about other
downstream tasks should not be drawn from this of course.

\section{A dynamical system perspective in the linear case}
Let us now consider, following~\citep{Littwin24,Tian21,Wang21}, the
linear case where $f_\theta(x)=Ax$, $f_\xi(y)=Cy$, and $g_\psi(z)=Bz$,
where $x\in\RR^m$, $A$ and $C$ are $n\times m$ matrices, with $n>m$,
$z \in \RR^n$, $B$ is an $n\times n$ matrix, and the vectors
$\theta$, $\xi$ and $\psi$ store row after row the coefficients of $A$,
$C$ and $B$.  We will from now on write $[xx^T]$ for $\EE_{x}\left[xx^T\right]$
and $[yx^T]$ for $\EE_{x,y}\left[yx^T\right]$. Recall that we consider, here, the (half) squared Euclidean distance for $l$ and $\Omega(\theta,\psi)=\lambda(\Vert \theta\Vert^2 + \Vert \psi \Vert^2)/2$.

\subsection{The SG and EMA algorithms as dynamical
  systems\label{sec:discr}}
An algorithm such as the SG and EMA procedures can also be viewed as a discrete
dynamical system. In this case, the dynamics are driven by the
gradients of the function $F$ with respect to the parameters being
optimized, {\em i.e.}, the matrices $A$, $B$ and $C$. Note: {\em We
  assume from now on that $\alpha_t$ is a constant $\alpha\neq 1$}.
For completeness, we rederive in the rest of this section several
results about the dynamics of the SG and EMA algorithms
(Lemmas~\ref{lemma:discdyn} to~\ref{lemma:th1}) already known
from~\citep{Tian21}.  As noted earlier, however, our proofs, given in
appendix, do not rely on the assumptions required by the original
ones, for example that the two views of the data be drawn from the
same distribution conditioned on the data, or that the eigenvalues of
certain PSD matrices be bounded away from zero, whose validity is
difficult to guarantee in practice.

\begin{lemma}
  The discrete dynamics of the EMA algorithm in the linear case are
  given by
  \begin{equation}
    \left\{\begin{array}{l}
             A_{t}=A_{t-1}-\mu_t(B_{t-1}^TR(A_{t-1},B_{t-1},C_{t-1})+\lambda A_{t-1}),\\
             B_t= B_{t-1}-\nu_t(R(A_{t-1},B_{t-1},C_{t-1})A_{t-1}^T+\lambda B_{t-1}),\\
             C_t=\alpha C_{t-1}+(1-\alpha)A_t,
           \end{array}\right.
\label{eq:lingrad}
\end{equation}
where $R(A,B,C)\eqdef BA[xx^T]-C[yx^T]$.
\label{lemma:discdyn}
\end{lemma}
 The dynamics for the SG algorithm are obtained by taking $C=A$ in
Eq.~(\ref{eq:lingrad}) and not using the $C$ update.
The stationary points of SG must satisfy the $p+q$  equations in
$p+q$ unknowns defined by $\bar{P}=0$ and $\bar{Q}=0$, with $p=nm$ and
$q=n^2$, whose solutions include $A=0$ and $B=0$. {\em In particular,
 unlike the  nonlinear case, it is {\em a priori} possible
  in the linear setting for a limit point of the SG or EMA algorithm
  to be the degenerate global minimum associated with $A=0$
  and $B=0$.}\footnote{The degenerate solution where $f_\theta$
  is a nonzero constant and $g_\psi$ is the identity does not occur in
  the linear case since $Ax$ constant implies $A=0$.}
The following lemma follows easily from  Eq.~(\ref{eq:lingrad}) and is reminiscent of prior results on gradient descent for deep linear networks \citep{baldi1989neural}.

\begin{lemma}
 When $\lambda>0$,  the limit points of the SG algorithm, if they exist, satisfy
$B^TB=AA^T$.
\label{lemma:aat-btb}
 \end{lemma}

We switch now to a continuous dynamical system perspective, following~\citep{Littwin24,Tian21,Wang21}, to simplify the analysis. The continuous version of Eq.~(\ref{eq:lingrad}) is
\begin{equation}
\left\{\begin{array}{l}
\dot{A}=-(B^TR(A,B,C)+\lambda A),\\
\dot{B}=-(R(A,B,C)A^T+\lambda B),\\
\dot{C}=(1-\alpha)(A-C).\end{array}\right.
     \label{eq:doteqs}
\end{equation}
where $R(A,B,C)$ is the same as in Lemma~\ref{lemma:discdyn}. At a limit point, we have in
addition $\dot{A}=0$, $\dot{B}=0$ and multiplying the first equation
on the right by $A^T$ and the second one on the left by $B^T$, then
subtracting the results yields
$B^TB=AA^T$.
Given Eq.~(\ref{eq:doteqs}) we can now rederive two results
from~\citep{Tian21}.

\begin{lemma}
Given the dynamical system associated with the gradient flow for the
original objective function of Eq.~(\ref{eq:mainobj}), the matrix
$A$ always converges to zero.
\label{lemma:dyn2}
\end{lemma}

This is Theorem 2 in~\citet{Tian21}, but without the assumptions made
in the original proof.
This result implies that there is no
other limit point than the global trivial minima where $A=0$ and $B$
can assume any value. \citet{Tian21} (and~\citep{Littwin24,Wang21})
interpret Eq.~(\ref{eq:doteqs}) as defining a {\em gradient
  flow}.\comment{\footnote{Most of their analysis is actually
    concerned with the EMA procedure rather than the SG algorithm
    considered here but this does not change our conclusions.}} Note
that it might be more appropriate to see these equations as defining a
general {\em flow} instead, that is, a first-order ordinary equation
whose integral curves are the output of the algorithm since we have
shown that their corresponding vector fields are not the gradients of a well
defined function by Proposition~\ref{prop:noptim}.  We can now state in our setting Theorem 1
from~\citep{Tian21} as follows.

\begin{lemma}
Under the SG or EMA dynamics, the difference of the two matrices $B^TB$ and
$AA^T$ tends to zero as $t$ tends to infinity.
\label{lemma:th1}
\end{lemma}

\subsection{Characterizing the equilibria of the SG and EMA algorithms}
\begin{proposition}
  Assuming that the matrix $A$ has maximal rank $m<n$ at equilibria of the SG
   algorithm, these equilibria are the $(A,B)$ pairs such that
  \begin{equation}
    ([xx^T]S+\lambda\text{Id})(S[xx^T]+\lambda\text{Id})=[xy^T]S[yx^T],\,\,
    \text{where}\,\,
    S=A^TA.
  \label{eq:Achar}
\end{equation}
and
\begin{equation}
  B=A[yx^T]A^TW^{-1}\,\,\text{where}\,\,
  W=A[xx^T]A^T+\lambda\text{Id}.
  \label{eq:Bchar}
\end{equation}
Assuming again that the matrix $A$ has maximal rank $m<n$ at the equilibria of
the EMA algorithm, these equilibria are the $(A,B,C)$ triples where
$A$ verifies Eq.~(\ref{eq:Achar}), $B$ is given by
Eq.~(\ref{eq:Bchar}) and $C=A$.
In both cases, the $(A,B)$ pairs associated with equilibria  also
verify $B^TB=AA^T$.
\label{prop:Eqchar}
\end{proposition}

Equation~(\ref{eq:Achar}) is a system of $m(m+1)/2$
quadratic equations in the $m(m+1)/2$ independent entries of the
symmetric matrix $S$, which is positive definite since $A$ is assumed to
have maximal rank. Such a system admits in general at most
$2^{m(m+1)/2}$  solutions, and we obtain the
following result.

\begin{corollary}
  Let $K$ denote the number of distinct real solutions $S_k$
  ($k=1,\ldots,K$) of Eq.~(\ref{eq:Achar}) such that $S_k$ is positive
  definite and, for $k=1,\ldots,K$, let $\sqrt{S}_k$ denote the
  unique positive definite square root of $S_k$. The equilibria of the SG algorithm can
  be decomposed into $K$ sub-varieties of $\RR^{n\times
    m}\times\RR^{m\times m}$ formed by pairs $(A,B)$ such that
  $A$ belongs to
  \begin{equation}
    \mathcal{A}_k=\{A\in\RR^{n\times m},\,\,A^TA=S_k\}=
    \{U\sqrt{S}_k,\,\,\text{where}\,\, U\in\RR^{n\times
      m}\,\,
    \text{and}\,\, U^TU=\text{Id}\},
    \end{equation}
    and $B$ satisfies Eq.~(\ref{eq:Bchar}). The equilibria $(A,B,C)$
    of the EMA algorithm can be characterized in a similar way to $A$ in $\mathcal{A}_k$, $B$ satisfying Eq.~(\ref{eq:Bchar}) and $C=A$.
    \label{corollary1}
  \end{corollary}

One might ask whether such maximal rank solutions exist in the first place. The following proposition, whose proof is provided in the appendix, leverages Brouwer's fixed point theorem and the implicit function theorem to establish that, in general, they do exist.
\begin{proposition}\label{prop:genericity_full_rank}
Let $\mathbb{P}_0$ be some data distribution.
	For any positive $\epsilon$, there always exists a perturbed data distribution $(\tilde{x},\tilde{y})\sim \mathbb{P}_{\epsilon}$ that is $\epsilon$-close to $\mathbb{P}$ (e.g., in the 2-Wasserstein distance sense), so that there exist equilibria of the SG and EMA algorithms with $A$ having maximal rank $m$ for $\lambda$ small enough.
\end{proposition}

Using classical results on the dynamics of  differential
equations will now allow us to prove our last
result. Let us first define the stable equilibria
of such a dynamical system~\citep{Arnoldiff92}.

\begin{definition}
  An equilibrium for the dynamical system  $\dot{z}=v(z)$,
  where $v$ is a smooth  field over $\RR^d$, is a point $e$
 where $v(e)=0$. An equilibrium $e$ is called (Lyapunov)
  stable when solutions of the differential equation with initial
  values close  to $e$ converge uniformly to a nearby point.  A
  stable equilibrium is said to be asymptotically stable when any
  solution started  close to $e$ converges to $e$.
  \end{definition}
\begin{theorem}~\citep[Chap. 3]{Arnoldiff92}
Consider a dynamical system  $\dot{x}=v(x)$ whose dynamics can be
approximated  by the linear operator $J$:
$v(x)=Jx+O(\Vert x\Vert^2)$.
A sufficient condition for an equilibrium to be
asymptotically
stable is that all eigenvalues of $J$ have a negative real part.\label{th:Arnold}
\end{theorem}

Armed with this classical result, we prove in the appendix the following proposition.
\begin{proposition}
  The equilibria of the dynamical system associated with
  the SG or the EMA procedure, if any, are, in general, asymptotically stable.
  \label{prop:dynamics}
\end{proposition}

Note that Proposition~\ref{prop:dynamics} does not imply that a
nontrivial equilibrium exists or that the dynamics converge to such an
equilibrium. It is, however, valid even in the case where $\alpha=1$.

\begin{figure}
  \vspace{-2em}
  \centerline{%
    \includegraphics[height=0.24\linewidth]{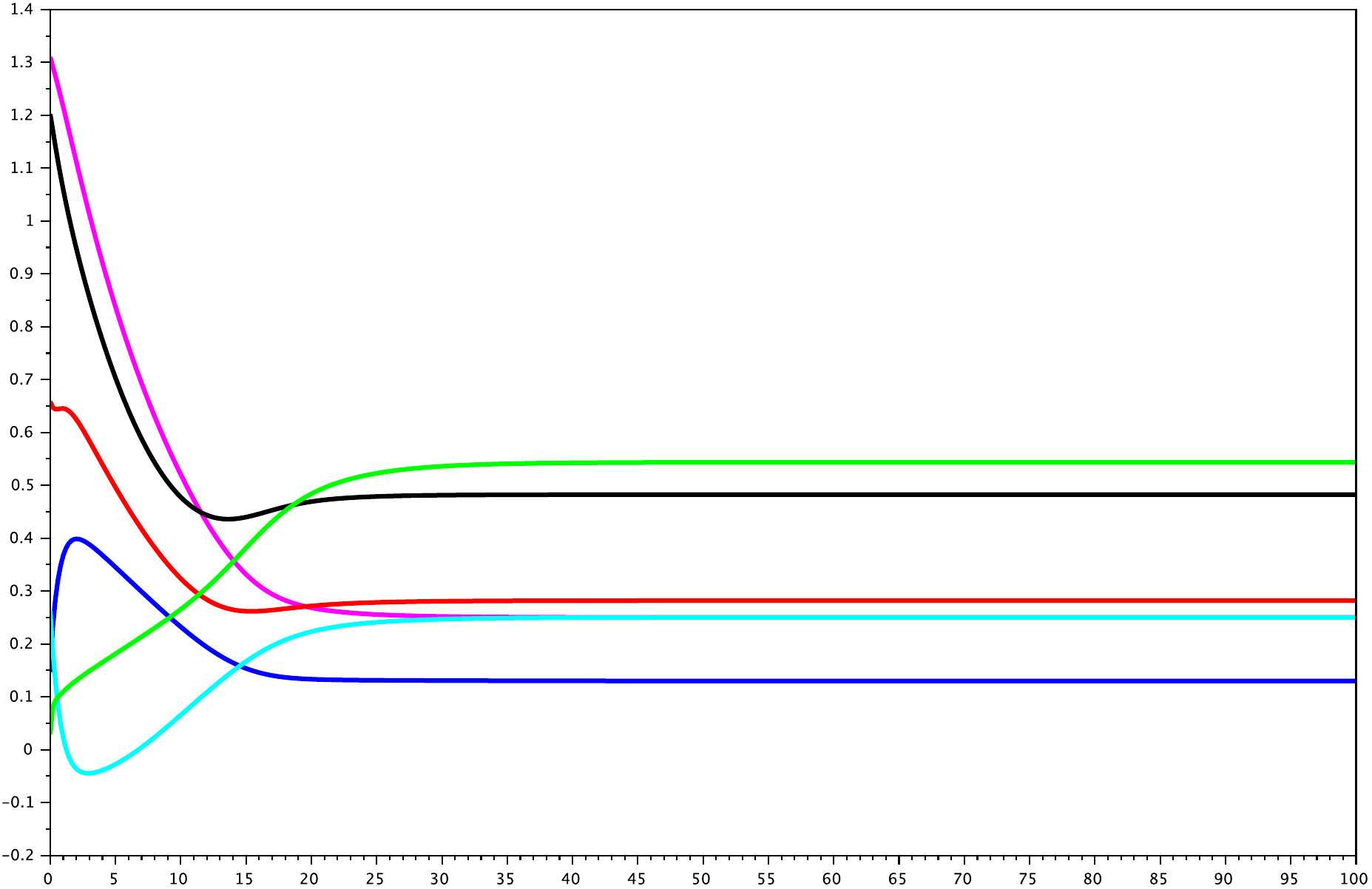}
    \includegraphics[height=0.24\linewidth]{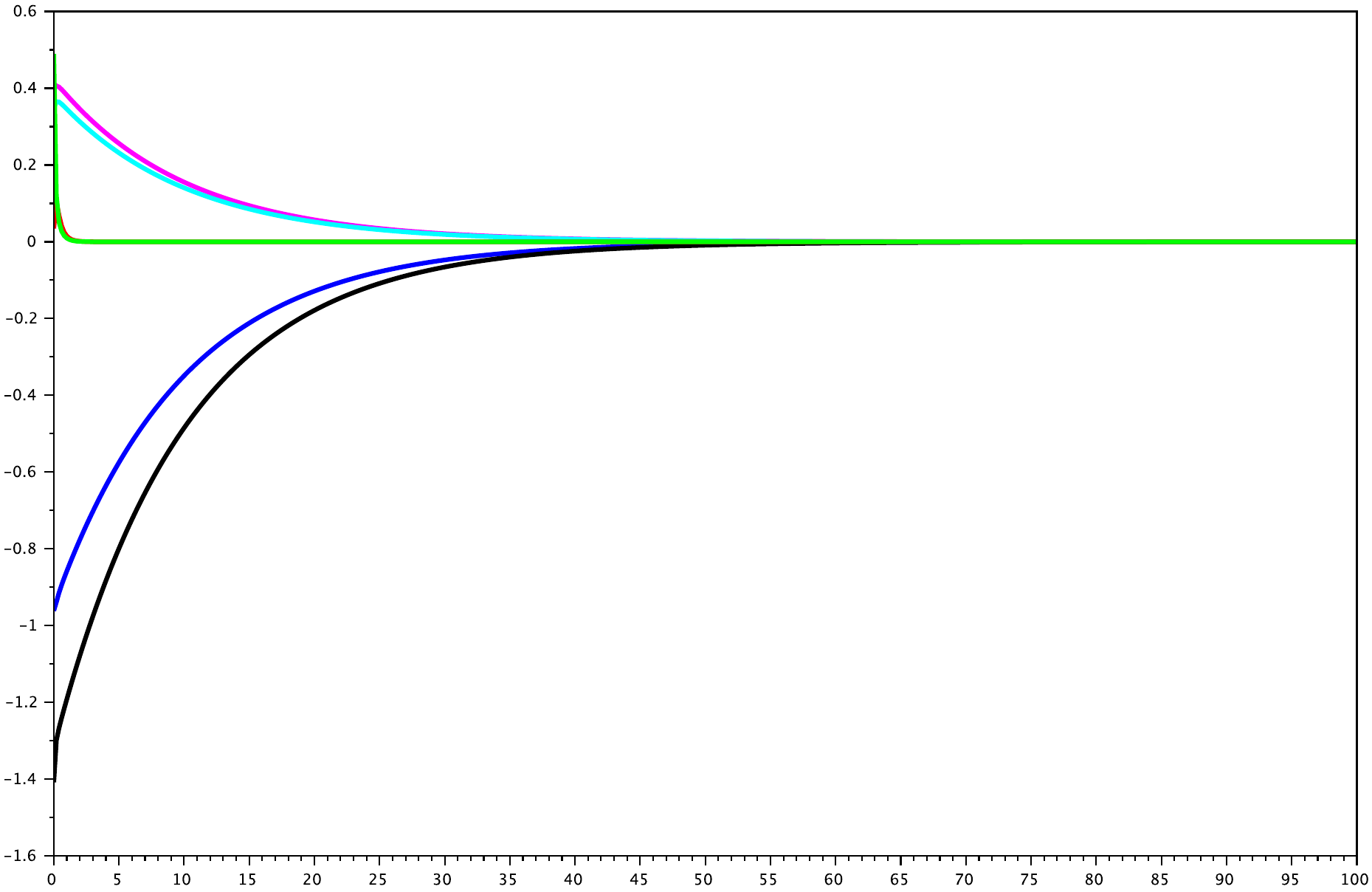}
    \includegraphics[height=0.24\linewidth]{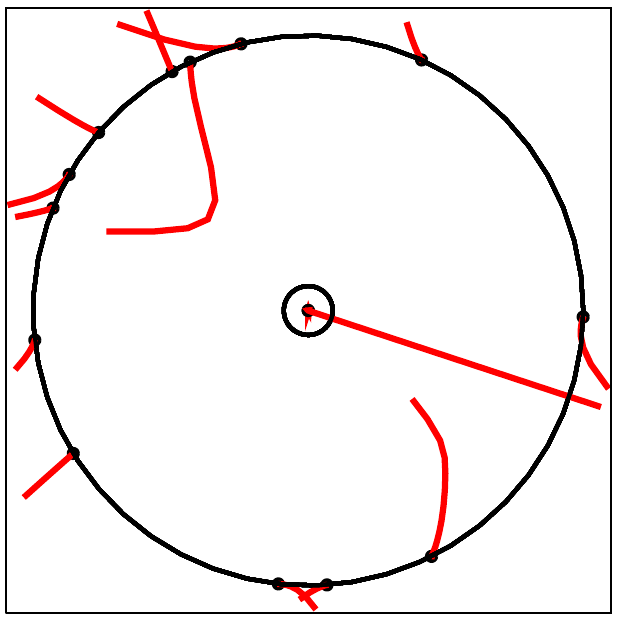}}
  \centerline{%
    \includegraphics[height=0.24\linewidth]{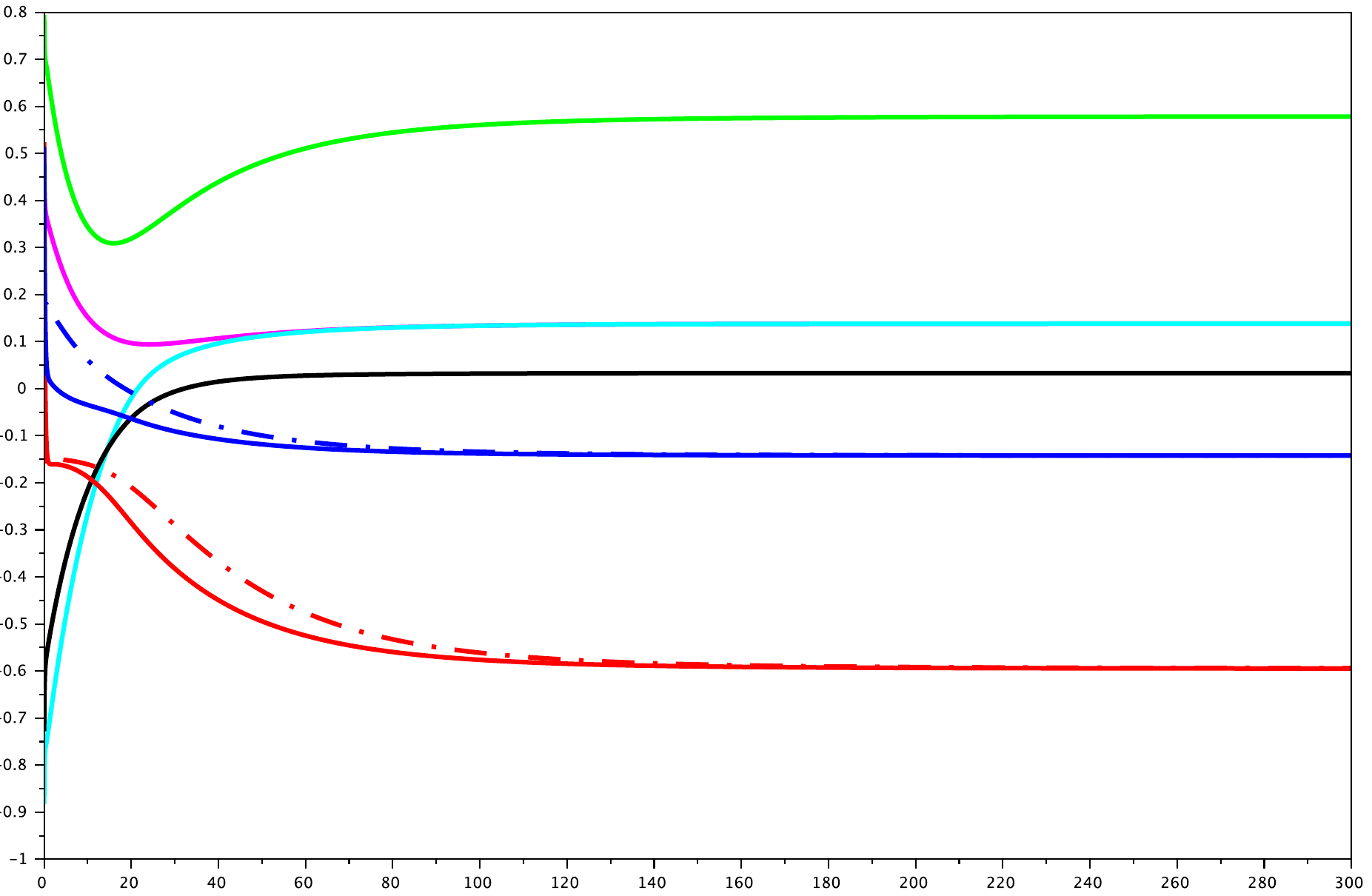}
    \includegraphics[height=0.24\linewidth]{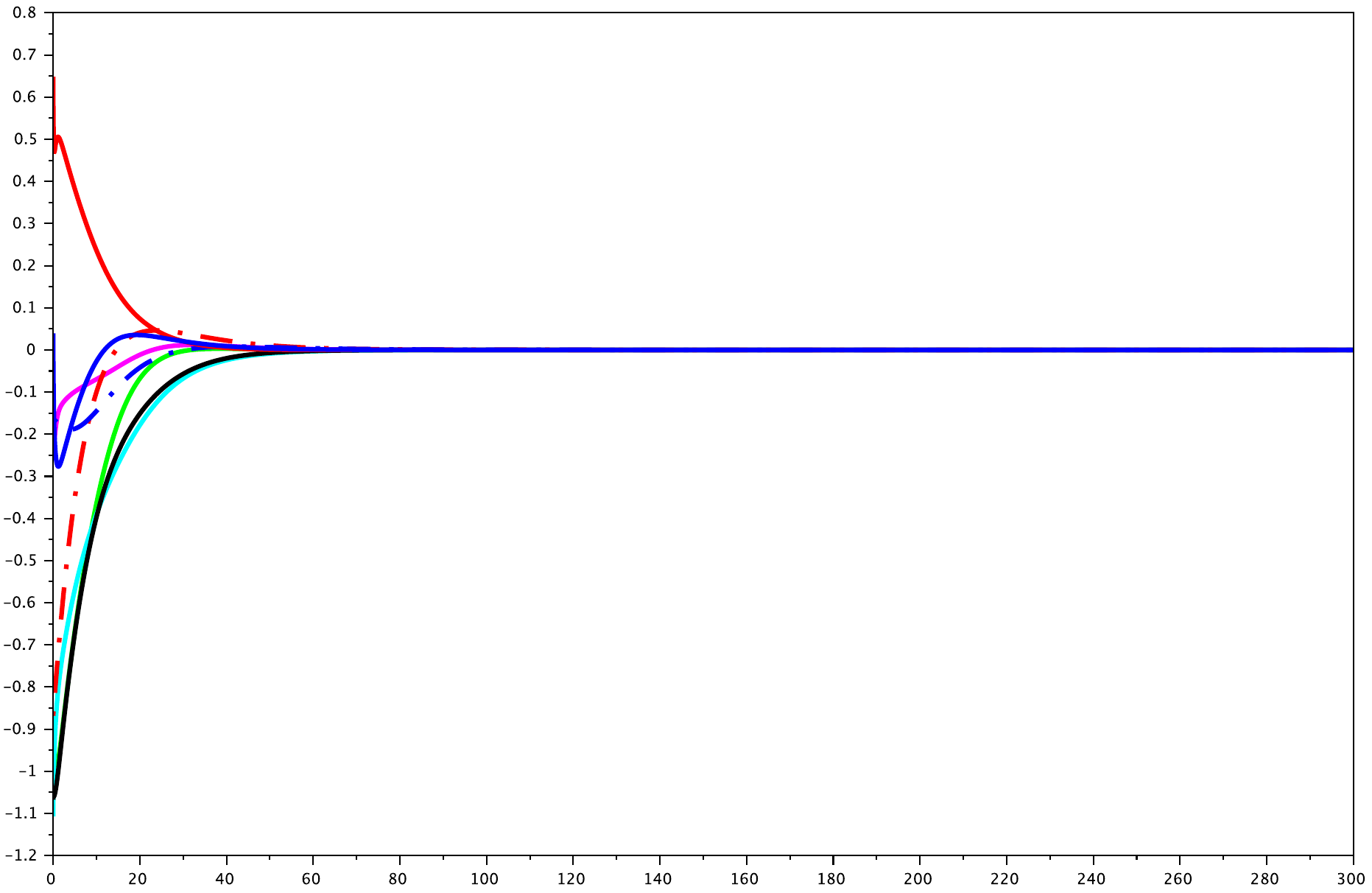}
    \includegraphics[height=0.24\linewidth]{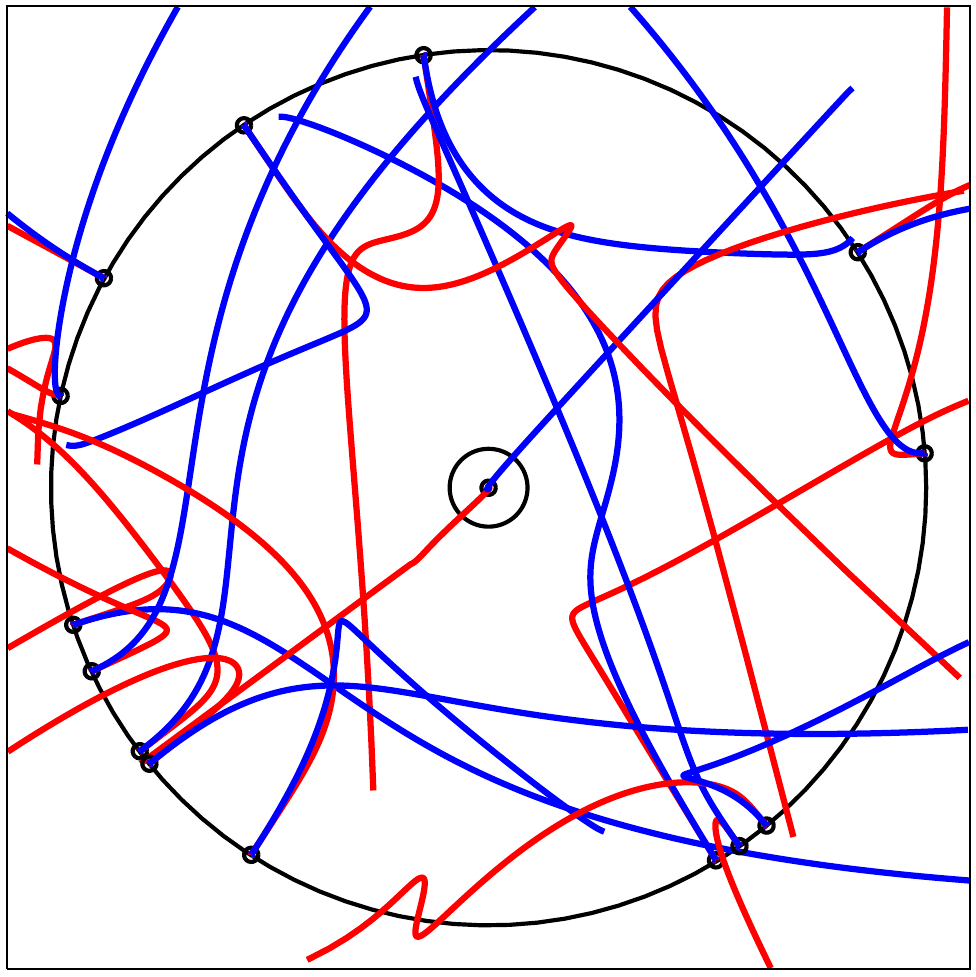}}
  \caption{\small Integration paths for the SG (top) and EMA (bottom)
    procedures. In the top (resp. bottom) part of the figure, the left
    diagram shows a sample path for the 8 (resp. 6) coefficients of
     $a,B,c$ (resp. $a,B$), the central one shows
   a path where all parameters converge toward $0$, and the
    right diagram shows 15 trajectories in the $a$ (resp. $a,c$)
    plane, starting from  random locations. One of the
    trajectories has the origin as a limit point.}
  \vspace{-1em}
  \label{fig:integ}
\end{figure}

\subsection{Experiments with synthetic data\label{sec:synthexp}}
For these experiments, we consider the case of a scalar input space
($m=1$) since $A$ is now a vector, which simplifies visualizations and
numerical simulations, and denote by $\rho$ and $\tau$ the scalars
$[xx^T]$ and $[yx^T]$. To emphasize that these quantities are vectors,
let us use $a$ for $A$ and $c$ for $C$. In this simplified
setting, equilibria are characterized by the following proposition.

\begin{proposition}
  In the case $m=1$, a necessary and sufficient condition for the
  existence of nonzero equilibria of the SG and EMA algorithms is
  that $\Delta=\tau^2-4\rho\lambda\ge 0$. When this condition is
  satisfied, these equilibria are the pairs $(a,B)$ such that $a$ lies
  on either one of the hyperspheres $S_1$ and $S_2$ of $\RR^n$
  centered at the origin with radii $r_1=(|\tau|-\sqrt\Delta)/2\rho$
  and $r_2=(|\tau|+\sqrt\Delta)/2\rho$, and $B$ verifies
  Eq.~(\ref{eq:Bchar}). The equilibria associated with $S_2$ are
  asymptotically stable, but those associated with $S_1$ are saddle
  points. The equilibria $(a,B,c)$ of the EMA algorithm and their
  stability can be characterized in a similar fashion, with the
  additional condition that $c=a$.
  \label{prop:dynm1}
\end{proposition}

Proposition~\ref{prop:dynm1} shows that the case $m=1$ is both
illustrative of the general case because the two spheres $S_1$ and
$S_2$ are just the algebraic varieties identified in
Proposition~\ref{prop:Eqchar} and Corollary~\ref{corollary1}, and
extremely nongeneric since the equilibria associated with $S_1$
are saddle points, which never happens when $m>1$. See the proof
of Proposition~\ref{prop:dynm1} in the appendix for an explanation
of this phenomenon.

\mypar{Results and conclusions.}  Figure~\ref{fig:integ} shows sample
trajectories obtained by numerical simulations for $n=2$, so $a$ and
$c$ are points in the plane, and equilibria are located at the origin
and on two circles centered at the origin with radii $r_1<r_2$. We
have (arbitrarily) taken $\rho=3$, $\tau=2$ and $\lambda=0.1$ with
$T=300$ time steps. For the SG algorithm (Figure~\ref{fig:integ},
top), we show on the left the evolution of $a$ (red and blue lines)
and $B$ (in other colors), and in the middle a case where all
coefficients converge to zero.  Fifteen trajectories initialized from
random positions and drawn in the $a$ plane are shown on the right.
Figure~\ref{fig:integ} (bottom) illustrates the EMA algorithm.  We use
$\alpha_t=0.9+0.1t/T$ in this experiment, so $\alpha=1$ at $t=T$,
following~\cite{BYOL20}.  Although $c$ is not guaranteed to converge
to $a$ in this case, it has done so in all our trials. This is
illustrated on the left, where the $a$ parameters are shown as solid
red and blue curves, and the $c$ parameters are shown as dashed red
and blue curves (the other 4 curves correspond to $B$). The center
part of the figure is an example where all parameters converge to
zero. Sample trajectories starting from various random positions are
shown on the right as red curves for $a$ and blue curves for $c$.
Although we have not been able to prove the convergence of the EMA
and SG algorithms so far, Figure~\ref{fig:integ} is typical of our
observations: In 10,000 trials with parameters drawn at
random,\footnote{Values for $\rho$, $\tau$ and $\lambda$ are
  respectively drawn uniformly in $[0,3]$, $[-1,1]$ and
  $[0.01,0.1]$. The coordinates of initial points are drawn from
  normal distributions with zero mean and unit variance.} the two
  algorithms always converge, 92.8\% (resp.  82.0\%) to a limit point
  on the outer circle $S_2$ for EMA (resp. SG), and the rest of the
  time to the origin (as noted in Section~\ref{sec:discr}, contrary to
  the generic nonlinear case, $A=0$ and $B=0$ are equilibria in the
  linear case). As expected, we have never observed convergence to the
  saddle points of $S_1$.
  As in the nonlinear case, the SG and EMA procedures do not appear to
minimize $\bar{E}$ or $\bar{F}$.  This is illustrated by
Figure~\ref{fig:bare}, where we plot the values of $\bar{E}(a,B)$ as a
function of time for some trajectories of the SG and EMA algorithms.
Although $\bar{E}(a,B)$ appears to steadily decrease in some cases, it
definitely does not in other cases.

\begin{figure}
  \vspace{-2em}
  \centerline{%
    \includegraphics[height=0.21\linewidth]{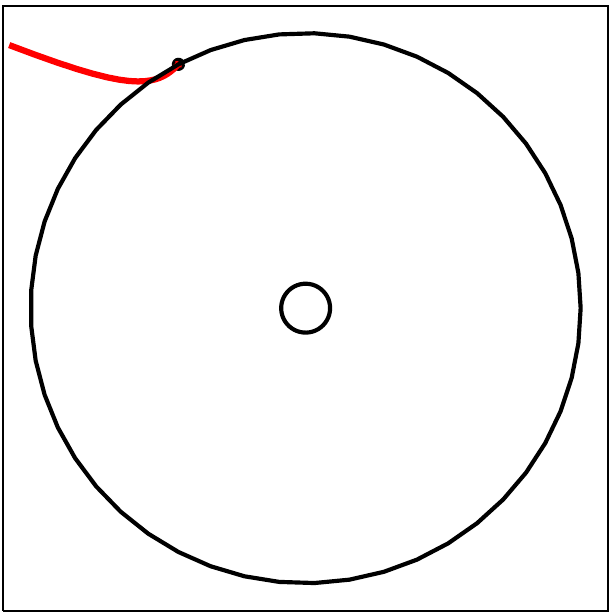}
    \includegraphics[height=0.21\linewidth]{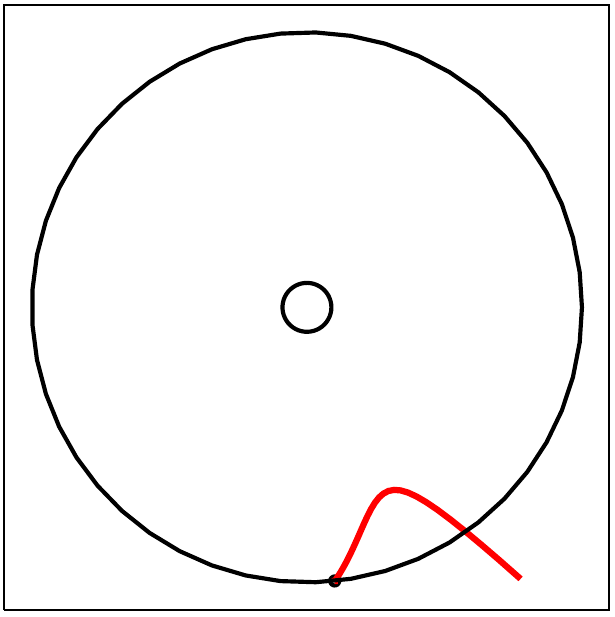}
    \includegraphics[height=0.21\textwidth]{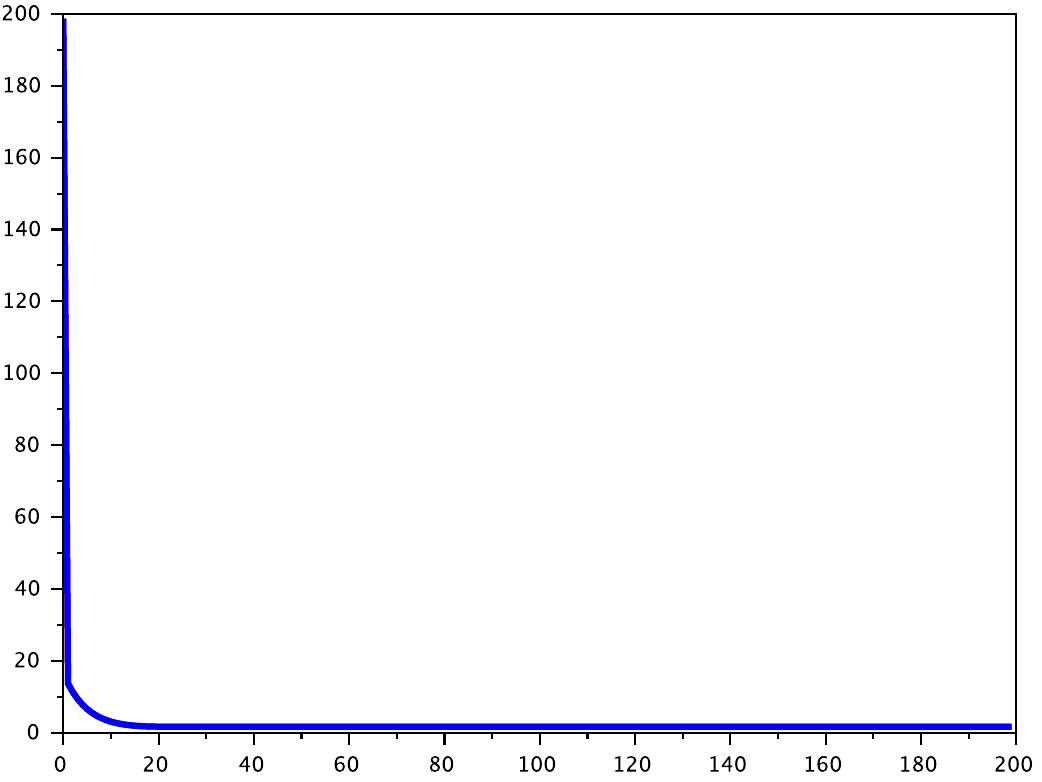}
    \includegraphics[height=0.21\linewidth]{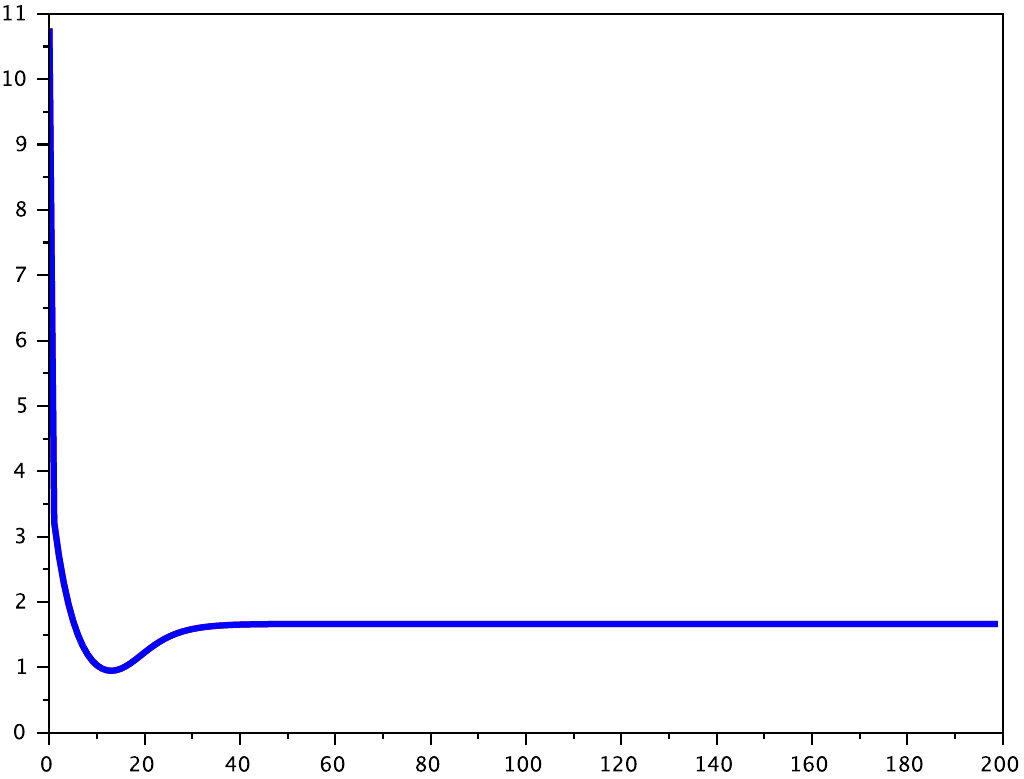}}
  \centerline{%
    \includegraphics[height=0.21\linewidth]{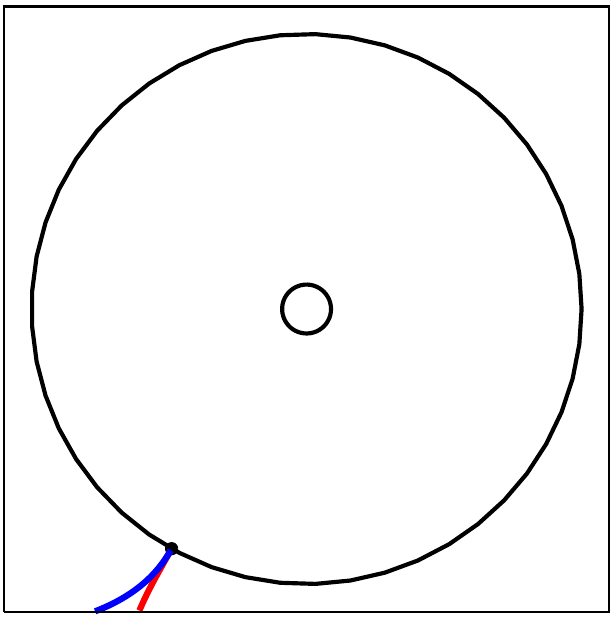}
    \includegraphics[height=0.21\linewidth]{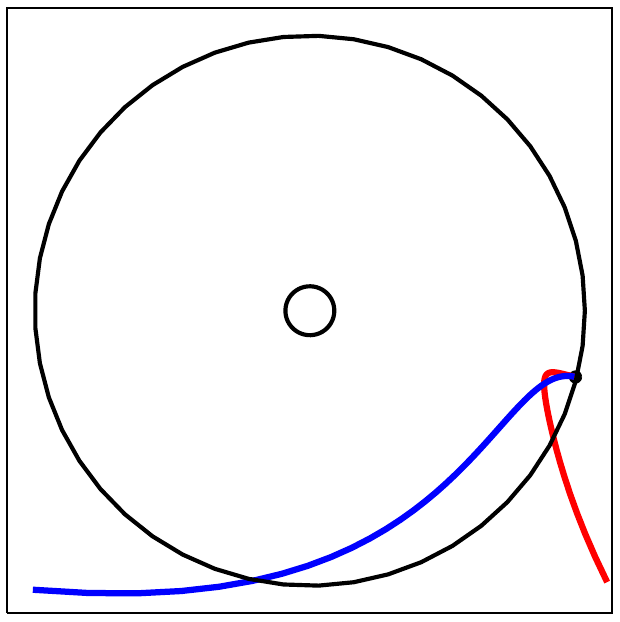}
    \includegraphics[height=0.21\textwidth]{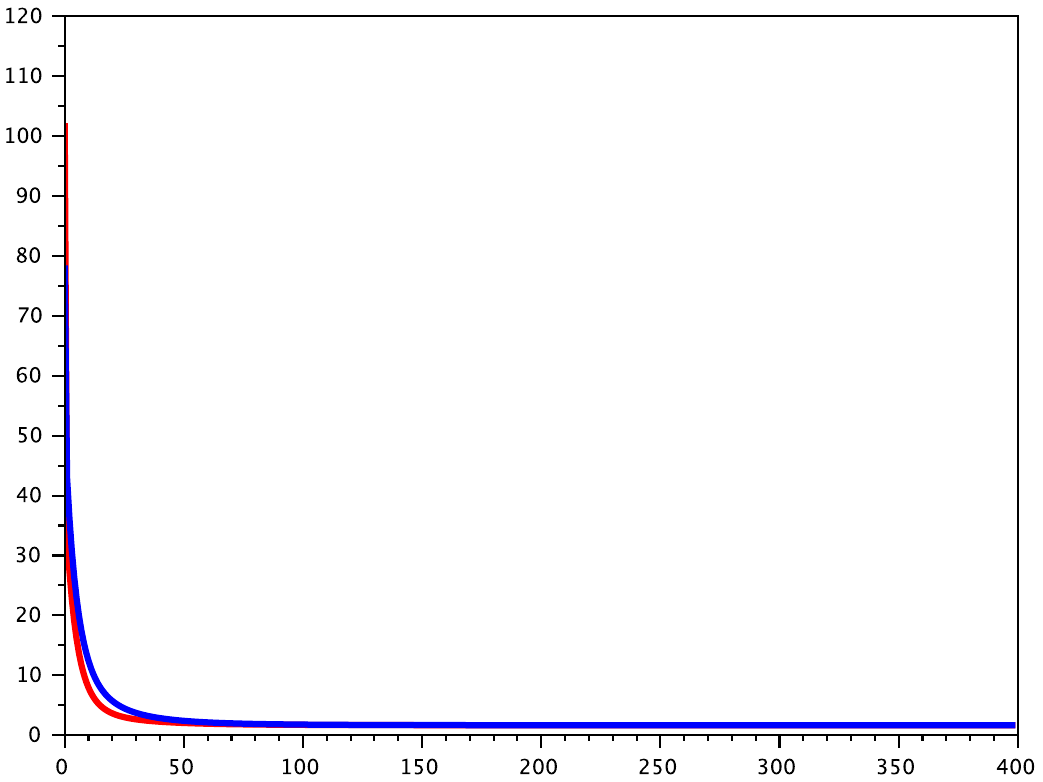}
    \includegraphics[height=0.21\linewidth]{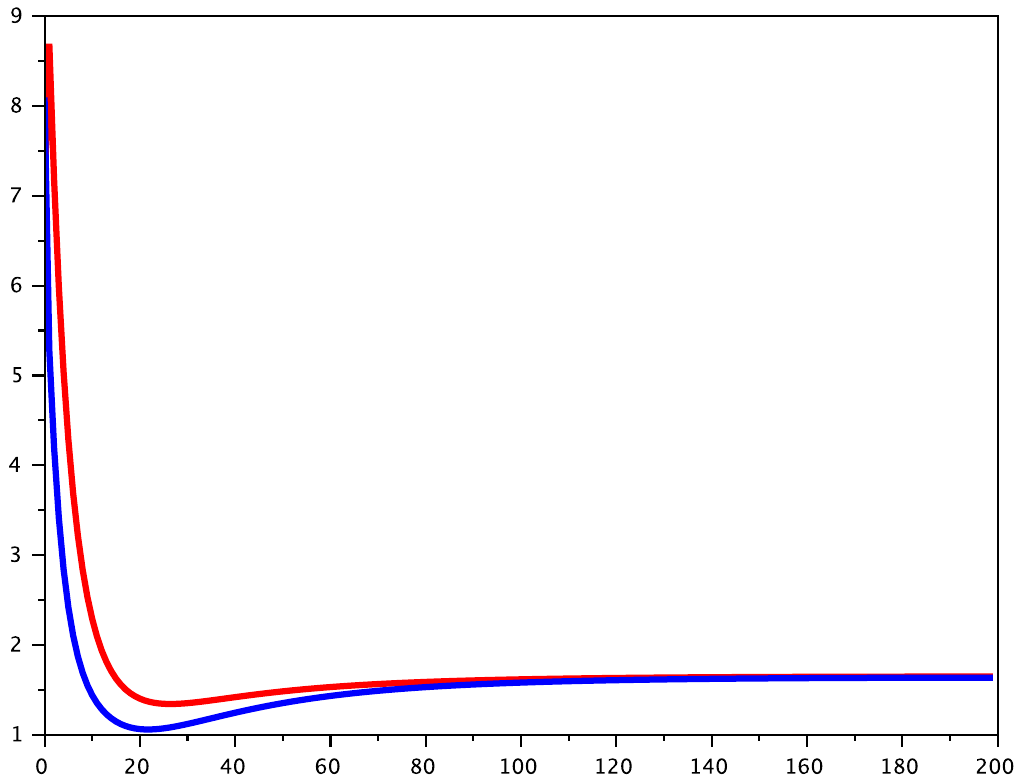}}
  \caption{\small Top: Two trajectories for the SG algorithm plotted
    in the $a$ plane (left) and the evolution of $\bar{E}(a,B)$ values
  along these trajectories (right). The error appears to decrease steadily
  with time in the first case, but not in the second one. Bottom: A
  similar diagram for the EMA algorithm. In this case we show in red
  the curve $\bar{E}(a,B)$ and in blue the curve $\bar{F}(a,B,c)$.}
  \vspace{-1em}
  \label{fig:bare}
\end{figure}

\section{Conclusions}
We have shown that the SG and EMA algorithms are not, in general,
proper {\em optimization procedures}: In particular, they do not
minimize {\em any} well defined objective function. On the other hand,
they do not lead in general to collapse when they converge, and they
enjoy interesting properties as {\em dynamical systems} since in the
linear case any nontrivial limit point is asymptotically stable thus
will not devolve into a trivial one by longer training. This point is
important in practice since the SG and EMA training procedures
empirically give good results, as shown for example
in~\citep{Bardes24,BYOL20,SimSiam21}, and appear, in the general
nonlinear
case, to prevent falling
into the degenerate global minima they are designed to avoid. But
then, what is it they really learn in the classical sense of the word?
Much work remains to be done.

\newpage


\subsection*{Ethics statement}
This work explores the theoretical foundations of two widely used
methods in self-supervised learning, the stop-gradient and exponential
moving average algorithms. Self-supervised representation learning has
already had a significant impact in various applications, including
natural language processing, computer vision, and robotics. We do not
anticipate particular risks of this work, but a deeper understanding of
self-supervised learning may allow the development of more performant and
scalable algorithms, with potential increased impact on the
aforementioned fields, which are not exempt of risks of misuse.

\subsection*{Reproducibility statement}
We have included in the main text all the necessary details to
understand and reproduce our experimental results, both for the
real-world data experiments, in Section~\ref{sec:realexp}, and for the
synthetic data in Section~\ref{sec:synthexp}.  For the real-world data
experiments, we use the publicly available code of
V-JEPA~\citep{Bardes24}, which we slightly modify as explained in
Section~\ref{sec:realexp}.

\input{acks.tex}

\clearpage

\bibliographystyle{plainnat}
\bibliography{kriegs}

\clearpage

\section{Proofs}

We will use, when convenient, the notation $f(\theta, x)$ for $f_{\theta}(x)$ and $g(\psi, z)$ for $
g_{\psi}(z)$. For a given matrix $A$ we denote by $vec(A)$ the vector collecting all the components of the matrix $A$.

\subsection*{Proof of Proposition~\ref{prop:first}}
As noted before, $\bar{E}(\theta,\psi)=\bar{F}(\theta,\psi,\theta)$ for any
values of $\theta$ and $\psi$. However, there is no a priori reason
for a limit point of  SG (if and when it
converges) to be a critical point of $\bar{E}$. Let us write
\begin{equation}
  F(\theta,\psi,\xi,x,y)=l[u(\theta,\psi,x),v(\xi,y)]+\Omega(\theta,\psi),
\end{equation}
where $ u(\theta,\psi,x)=g_\psi\circ f_\theta(x)$ and $v(\xi,y)=f_\xi(y)$,
and denote by $J_\theta u(\theta,\psi,x)$ and
$J_\psi u(\theta,\psi,x)$ the $n\times p$ and $n\times q$ Jacobians of
$u(\theta,\psi,x)=g_\psi\circ f_\theta(x)$ with respect to $\theta$
and $\psi$.
We also introduce $l(\theta,\psi,\xi,x,y)=l[u(\theta,\psi,x),v(\xi,y)]$.
We have
\begin{equation}
  \left\{
\begin{array}{l}
           \nabla_\theta F=J_\theta
           u(\theta,\psi,x)^T
           \nabla_ul(\theta,\psi,\theta,x,y)+
\nabla_\theta\Omega(\theta,\psi),\\
           \nabla_\psi F=J_\psi u(\theta,\psi,x)^T\nabla_ul(\theta,\psi,\theta,x,y)+
\nabla_\psi\Omega(\theta,\psi),
         \end{array}\right.
       \label{eq:grads}
\end{equation}
and the means of these two gradients (also  gradients of
$\bar{F}$) should vanish at any limit point of  SG. But
\begin{equation}
  \left\{
\begin{array}{l}
  \nabla_\theta \bar{E}(\theta,\psi)=\nabla_\theta \bar{F}(\theta,\psi,\theta)
  +\EE_{x,y} [J_\theta f(\theta,y)^T\nabla_v l[u(\theta,\psi,x),v(\theta,x,y)]],\\
    \nabla_\psi \bar{E}(\theta,\psi)=    \nabla_\psi
           \bar{F}(\theta,\psi,\theta),
\end{array}
\right.
\end{equation}
where $J_\theta f(\theta,y)=J_\xi v(\theta,y)$ is the $n\times p$
Jacobian of $f$ with respect to $\theta$. There is a priori no reason
why the second term of the gradient of $\bar{E}$ with respect to
$\theta$, which depends on the data, should be zero at a critical
point of $\bar{F}$ and thus at a limit point of SG (and thus of EMA)
if such a point exists.

To establish this, we consider a counter example in  the linear case where $f_{\theta}(x) = Ax$ and $g_{\psi}=Bz$ where $A$ is an $n\times m$ matrix and $B$ an $n\times n$ matrix with $n>m$. Here $\theta= vec(A)$ and $\xi=vec(B)$. Anticipating the result of Proposition~\ref{prop:genericity_full_rank}, we know that we can always find an arbitrarily small perturbation to the data distribution, call it $\mathbb{P}_{\epsilon}$ so that there exists a matrix $A^{\star}$ of maximal rank and a matrix $B^{\star}$, that are critical points of the dynamics. Set $\theta^{\star} = vec(A^{\star})$ and $\psi^{\star}= vec(B^{\star})$. Hence, $f_{\theta^{\star}}(x)$ and $g_{\psi^{\star}}(z)$ cannot be degenerate solutions of the EMA or SG dynamics.
Now, we wish to further establish that $\nabla_{\theta} \bar{E}_{\mathbb{P}_{\epsilon}}(\theta^{\star},\psi^{\star})$ is not $0$ in a generic sense.
Let us first express $\nabla_{\theta}\bar{E}_{\mathbb{P}_{\epsilon}}(\theta,\psi) := \mathbb{E}_{(\tilde{x},\tilde{y})\sim \mathbb{P}_{\epsilon}}[\nabla_{\theta} E(\theta,\psi, \tilde{x},\tilde{y})]$  in terms of the matrices $A^{\star}$ and $B^{\star}$ and $\nabla \bar{F}_{\mathbb{P}_{\epsilon}}(\theta^{\star},\psi^{\star},\theta^{\star}) := \mathbb{E}_{(\tilde{x},\tilde{y})\sim \mathbb{P}_{\epsilon}}[\nabla_{\theta} F(\theta,\psi,\theta, \tilde{x},\tilde{y})]$:
\begin{align*}
\nabla_{\theta}\bar{E}_{\mathbb{P}_{\epsilon}}(\theta^{\star},\psi^{\star})  &= \nabla \bar{F}_{\mathbb{P}_{\epsilon}}(\theta^{\star},\psi^{\star},\theta^{\star}) +   A^{\star}[\tilde{y}\tilde{y}^{\top}]- B^{\star}A^{\star}[\tilde{x}\tilde{y}^{\top}]\\
    &= A^{\star}[\tilde{y}\tilde{y}^{\top}]- B^{\star}A^{\star}[\tilde{x}\tilde{y}^{\top}].
\end{align*}
Note that $\nabla_{\theta} \bar{F}_{\mathbb{P}_{\epsilon}}(\theta,\psi,\theta)$ and  $\nabla_{\psi} \bar{F}_{\mathbb{P}_{\epsilon}}(\theta,\psi,\theta)$ are independent of the matrix $[\tilde{y}\tilde{y}^{\top}]$ by virtue of the expression:
\begin{align*}
    \nabla_{\theta} \bar{F}_{\mathbb{P}_{\epsilon}}(\theta,\psi,\theta) &= B^{\top}(BA [\tilde{x}\tilde{x}^{\top}]- A[\tilde{y}\tilde{x}^{\top}]) + \lambda A\\
    \nabla_{\psi} \bar{F}_{\mathbb{P}_{\epsilon}}(\theta,\psi,\theta)&= BA[\tilde{x}\tilde{x}^{\top}]A^{\top}- A[\tilde{y}\tilde{x}^{\top}]A^{\top} + \lambda B.
\end{align*}
Therefore, the critical points $(A^{\star},B^{\star})$ are independent of the matrix $[\tilde{y}\tilde{y}^{\top}]$. We treat two cases: either $\nabla_{\theta}\bar{E}_{\mathbb{P}_{\epsilon}}(\theta^{\star},\psi^{\star})\neq 0$, and there is nothing to prove, or  $\nabla_{\theta}\bar{E}_{\mathbb{P}_{\epsilon}}(\theta^{\star},\psi^{\star})=0$. In this case, we will construct a second perturbed data distribution $\mathbb{Q}_{\epsilon}$ from $\mathbb{P}_{\epsilon}$ as follows:
\begin{align}
    x' = \tilde{x}\qquad y' = \tilde{y} + \epsilon^{\frac{1}{2}}z,
\end{align}
where $(\tilde{x},\tilde{y})$ are samples from $\mathbb{P}_{\epsilon}$ and $z$ is a standard centered gaussian independent of $\tilde{x}$ and $\tilde{y}$. It is easy to check that $[x'(x')^{\top} ] = [\tilde{x}\tilde{x}^{\top} ]$  and that $[y'(x')^{\top} ] = [\tilde{y}\tilde{x}^{\top} ]$ so that the critical points $A^{\star}$ and $B^{\star}$ remain unchanged. On the other hand $[y'(y')^{\top}] = [\tilde{y}\tilde{y}^{\top}] + \epsilon \text{Id}$. It follows, under such perturbed distribution $\mathbb{Q}_{\epsilon}$ that:
\begin{align*}
\nabla_{\theta}\bar{E}_{\mathbb{Q}_{\epsilon}}(\theta^{\star},\psi^{\star})
&=  \nabla \bar{F}_{\mathbb{Q}_{\epsilon}}(\theta^{\star},\psi^{\star},\theta^{\star}) + A^{\star}[y'(y')^{\top}] - B^{\star}A^{\star}[[x'(y')^{\top}]\\
&= A^{\star}[\tilde{y}\tilde{y}^{\top}] - B^{\star}A^{\star}[\tilde{x}\tilde{y}^{\top}] + \epsilon A^{\star}\\
&= \epsilon A^{\star}\neq 0.
\end{align*}
Hence, we have established that one can always find a small perturbation $\mathbb{Q}_{\epsilon}$ to the data distribution so that the equilibrium is not degenerate and does not correspond to a critical point of the objective $\mathbb{E}_{(x',y')\sim \mathbb{Q}_{\epsilon}}[E(\theta,\psi,x',y')]$.
\hfill\qed

\subsection*{Proof of Proposition~\ref{prop:noptim}}
Define for simplicity, the following vector fields:
\begin{align}
P(\theta,\psi,x,y)&= \nabla_{\theta}F(\theta,\psi, \theta,x,y)\\
Q(\theta,\psi,x,y)&= \nabla_{\psi}F(\theta,\psi, \theta,x,y).
\end{align}
Let $\mathbb{D}$ be some probability distribution over data $(x,y)$ and define $\bar{P}_{\mathbb{D}}(\theta,\psi):= \mathbb{E}_{(x,y)\sim \mathbb{D}}\left[ P(\theta,\psi,x,y) \right]$ and $\bar{Q}_{\mathbb{D}}(\theta,\psi):= \mathbb{E}_{(x,y)\sim \mathbb{D}}\left[ Q(\theta,\psi,x,y) \right]$.
  According to Schwarz's integrability theorem, a {\em necessary}
  condition for $\bar{P}_{\mathbb{D}}$ and $\bar{Q}_{\mathbb{D}}$ to be the gradient field of a
  smooth scalar function is that their second-order cross derivatives
  be the transposes of each other.
We have
\begin{equation}
  \left\{
 \begin{array}{l}
           \displaystyle\frac{\partial
           P}{\partial\psi}(\theta,\psi,x,y)=H_{\theta,\psi}u(\theta,\psi,x)
           [u(\theta,\psi,x)-v(\theta,y)]
+J_\theta u (\theta,\psi,x)^TJ_\psi u(\theta,\psi,x),\\
           \displaystyle\frac{\partial
           Q}{\partial\theta}(\theta,\psi,x,y)=H_{\psi,\theta}u(\theta,\psi,x)[u(\theta,\psi,x)
           -v(\theta,y)]
+J_\psi u(\theta,\psi,x)^T [J_\theta
           u(\theta,\psi,x)-J_\theta v(\theta,y)],
         \end{array}\right.
     \end{equation}
where $J_\theta v(\theta,y)$ denotes the $n\times p$ Jacobian of $v$ with
respect to $\theta$, which is equal to the Jacobian $J_\theta
f(\theta,y)$, and $H_{\theta,\psi}u$ (resp. $H_{\psi,\theta}u$) is
the $p\times q\times n$ (resp. $q\times p\times n$) tensor
associated with the second partial derivative of $u$ with respect
to $\theta$ and $\psi$ (resp. $\psi$ and $\theta$). The products
of these tensors by the vector $u-v$ yield $p\times q$ and $q\times
p$ matrices that are transposes of each other, and
we  obtain
\begin{equation}
\bar{\Delta}_{{\mathbb{D}}}(\theta,\psi)\eqdef\frac{\partial \bar{P}_{{\mathbb{D}}}}{\partial\psi}(\theta,\psi)  -
\frac{\partial \bar{Q}_{{\mathbb{D}}}}{\partial\theta}(\theta,\psi) ^T
=  \EE_{x,y} [ \underbrace{J_\theta f(\theta,y)^TJ_\psi g(\psi,f(\theta,x))}_{\Delta(\theta,\xi, x,y)}].
  \label{eq:integ}
\end{equation}
We wish to establish that $\bar{\Delta}_{\mathbb{D}}\neq 0 $ in a generic sense over all possible data distributions. This means that even if for some data distribution $\mathbb{D}_0$ we have $\bar{\Delta}_{\mathbb{D}_0}=0$, we can always construct a perturbed data distribution $\mathbb{D}_{\epsilon}$ that gets arbitrarily close to $\mathbb{D}_0$ as $\epsilon$ approaches $0$ and for which $\bar{\Delta}_{\mathbb{D}_{\epsilon}}\neq 0$. To this end, let's consider a case for which $\bar{\Delta}_{\mathbb{D}_0}=0$ for some $\mathbb{D}_0$, otherwise there is nothing to prove.

By Proposition \ref{prop:non_vanishing_delta} stated below, we know that the integrand $\Delta$ appearing in the expression of $\bar{\Delta}_{\mathbb{D}_0}$ is a function that is not identically $0$. Hence, there exists $\theta$, $\xi$, $x_0$ and $y_0$ so that $\Vert\Delta(\theta,\xi, x_0,y_0)\Vert_{F}\geq 2\eta $ for some positive $\eta$. Since both function $f$ and $g$ are assumed to be continuously differentiable, we know that $\Delta$ must be a continuous map.
Hence, there exists a positive radius $r>0$ small enough so that $\Vert\Delta(\theta,\xi, x,y)-\Delta(\theta,\xi, x_0,y_0)\Vert_{F}\leq \eta $ for any $(x,y)$ in a ball  $\mathcal{B}$ centered around $(x_0,y_0)$ of radius $r$. Consider uniform distribution $q$ over such ball. It is then easy to see that $\Vert \mathbb{E}_{(x,y)\sim q}[\Delta(\theta,\xi,x,y)] - \Delta(\theta,\xi,x_0,y_0)\Vert_F\leq \eta$, so that:
\begin{align}
	 \eta \leq \Vert \Delta(\theta,\xi,x_0,y_0)\Vert_{F}  -\Vert \mathbb{E}_{(x,y)\sim q}[\Delta(\theta,\xi,x,y)] - \Delta(\theta,\xi,x_0,y_0)\Vert_F.
\end{align}
Now, define the following perturbed distribution by convex mixture of $\mathbb{D}_0$ and $q$:
\begin{align}
	\mathbb{D}_{\epsilon} = (1-\epsilon)\mathbb{D}_0 + \epsilon q.
\end{align}
We will show that $\Vert \bar{\Delta}_{\mathbb{D}_{\epsilon}}\Vert_{F}\geq \epsilon \eta$. To see this, we first notice that:
\begin{align}
	\bar{\Delta}_{\mathbb{D}_{\epsilon}} = (1-\epsilon)\underbrace{\bar{\Delta}_{\mathbb{D}_{0}}}_{=0} + \epsilon \mathbb{E}_{(x,y)\sim q}[\Delta(\theta,\xi,x,y)].
\end{align}
Hence, by direct application of the triangle inequality, we have that:
\begin{align}
	\epsilon\eta \leq \epsilon \left(\Vert \Delta(\theta,\xi,x_0,y_0)\Vert_{F}  -\Vert \mathbb{E}_{(x,y)\sim q}[\Delta(\theta,\xi,x,y)] - \Delta(\theta,\xi,x_0,y_0)\Vert_F\right)\leq \Vert \bar{\Delta}_{\mathbb{D}_{\epsilon}} \Vert_F.
\end{align}
We have, therefore shown an arbitrarily small perturbation on the data distribution suffices to ensure that  $\bar{\Delta}_{\mathbb{D}_{\epsilon}}\neq 0$, which in turn implies that that a the corresponding vector fields $\bar{P}_{\mathbb{D}_{\epsilon}}$ and $\bar{Q}_{\mathbb{D}_{\epsilon}}$ are not the gradient of any scalar function.
\hfill\qed

The interested reader may wonder how the integrability condition of
Proposition~\ref{prop:noptim} translates in the linear case.  As shown below, a necessary condition for $\bar{P}$ and $\bar{Q}$ to be the
gradient field of a smooth function for linear encoders and predictors
is that $A[xy^T]=0$. This condition depends on the data and is not in
general satisfied by $A$, as expected.

The proof goes as follows: Given any $s\times t$
matrix $U$ with rows $u_i^T$ ($i=1,\ldots,s$) and vector $v$ in
$\RR^t$, let $(Uv)_i=u_i\cdot v$ denote the $i$th entry of $Uv$. We
have $\partial (Uv)_i/\partial u_i=v$ and
$\partial (Uv)_i/\partial u_j=0$ for any $j\neq i$, and it follows
that the two Jacobians $J_\theta(Ay)$ and $J_\psi(BAx)$ are
respectively the $n\times p$ and $n\times q$ matrices
\begin{equation}
\qmatrix{y^T & \ldots & 0^T\\
 \vdots & \ddots & \vdots\\
 0^T & \ldots & y^T}\,\,\text{and}\,\,
\qmatrix{x^TA^T & \ldots & 0^T\\
 \vdots & \ddots & \vdots\\
 0^T & \ldots & x^TA^T}.
\end{equation}
Substituting these values in Eq.~(\ref{eq:integ})  shows
that $\Delta(\theta,\psi)$ is a $p\times q$ block diagonal matrix
whose $n$ $(m\times n)$ diagonal blocks are all equal to
$[yx^T]A^T$. Transposing  $\Delta(\theta,\psi)=0$
concludes the proof.

We end this section with the main proposition used in the proof above.
\begin{proposition}\label{prop:non_vanishing_delta}
	Assume that the parametric encoder and predictor are not identically $0$, i.e. there exists $\theta_0, \psi_0, x_0$ and $z_0$ so that $f(\theta_0,x_0)\neq 0$ and $g(\psi_0,z_0)\neq 0$. Furthermore, assume both $f$ and $g$ have a final linear layer, In other words, they can be expressed in the following form:
	\begin{align}
		f(\theta, x) := A\phi(U,x),\qquad g(\psi,z) := B h(V,z)
	\end{align}
where $A$ and $B$ are $n\times k$ and $n\times d$ matrices for some positive integers $k$ and $d$, while $U$ and $V$ are parameters so that $\theta= vec(A,U)$ and $\psi = vec(B,V)$. Here, $\phi$ and $h$ are differentiable functions with values in $\mathbb{R}^k$ and $\mathbb{R}^{d}$. Furthermore, consider the quantity
\begin{align}
\Delta(\theta,\psi,x,y) := J_{\theta}f(\theta,y)^{\top} J_{\psi} g(\psi, f(\theta,x)),
\end{align}
where $J_{\theta}f$ and $J_{\psi} g$ denote the Jacobians of $f$ and $g$ w.r.t. $\theta$ and $\psi$. Then $\Delta(\theta,\psi,x,y)$ cannot be identically $0$.
\end{proposition}
\begin{proof}
	It suffices to show that the components $\tilde{\Delta}$ of $\Delta$ corresponding to the partial derivatives w.r.t. the linear parameters $A$ and $B$ are not identically $0$. Hence, by contradiction, we will assume that $\tilde{\Delta}$ vanishes everywhere.
	Hence,  $\tilde{\Delta}$ must vanish, when applied to any arbitrary perturbation matrices $\delta A$ and $\delta B$ of $A$ and $B$, so that the following holds:
\begin{align}\label{eq:vanishing_identity}
	\tilde{\Delta}(\theta,\psi, x,y)(\delta A, \delta B) = (\delta A\phi(U, y))^{\top} \delta B h(V, A \phi(U,x)) = 0.
\end{align}
The above identity is obtained by standard calculus. Furthermore, since $f$ and $g$ are non vanishing, there exists parameter values $U_0$ and $V_0$ so that $x\mapsto \phi(U_0,x)$ and $z\mapsto h(V_0,z)$ are not identically $0$. Let $y_0$ be a vector in $\mathbb{R}^m$ so that $\phi(U_0,y_0)\neq 0$. Moreover, fix any arbitrary vector $z\in \mathbb{R}^n$. There must exist a perturbation matrix $\delta A$ so that $z = \delta A \phi(U_0,y_0)$ (simply take  $\delta A = \frac{1}{\Vert \phi(U_0,y_0)\Vert^2} z \phi(U_0,y_0)^{\top}$). Therefore, by  Equation \ref{eq:vanishing_identity} it follows:
\begin{align}
	z^{\top} (\delta B h(V, A \phi(U_0,x))) = 0,
\end{align}
for any $z$ in $\mathbb{R}^n$ and any matrix $\delta B$. This directly implies that $ h(V, A \phi(U_0,x)) = 0$ for any $A$, $x$ and $V$. Furthermore, by choosing $V=V_0$,  $x=y_0$ and $A = \frac{1}{\Vert \phi(U_0,y_0)\Vert^2} z \phi(U_0,y_0)^{\top}$ for any arbitrary $z\in \mathbb{R}^n$, we get that: $h(V_0, z)=0$. This contradicts the fact that $z\mapsto h(V_0,z)$ is not identically $0$. Hence, we have shown that $\tilde{\Delta}$ is not identically $0$ which, a fortiori, implies that $\Delta$ is itself not identically null.
\end{proof}

\subsection*{Proof of Lemma~\ref{lemma:discdyn}}
We can rewrite $F$ in this case as
\begin{equation}
  \hspace{-2mm}\begin{array}{lcl}
    F(A,B,C,x,y)&=&\frac{1}{2}\Vert BAx-Cy\Vert ^2+
                              \frac{\displaystyle\lambda}{\displaystyle
                              2}(\Vert A\Vert _F^2+\Vert B\Vert _F^2),\\
&=&\frac{1}{2}x^TA^TB^TBAx-y^TC^TBAx+\frac{1}{2}y^TC^TCy
+
                              \frac{\displaystyle\lambda}{\displaystyle
                              2}(\Vert A\Vert _F^2
  +\Vert B\Vert _F^2).
  \end{array}
\end{equation}

We know from Eqs.~[70] and [82] in {\em the matrix
  cookbook}~\citep{matcook12} that
\begin{equation}
\left\{  \begin{array}{l}
    \frac{\displaystyle\partial}{\strut\displaystyle\partial
    U}(a^TUb)=ab^T,\\
    \frac{\displaystyle\partial}{\strut\displaystyle\partial
    U}(b^TU^TVUc)=
    V^TUbc^T+VUcb^T.
         \end{array}\right.
       \label{eq:identities}
\end{equation}
It follows that:
\begin{equation}
\left\{\begin{array}{l}
\frac{\displaystyle\partial F}{\displaystyle\partial
           A}(A,B,C,x,y)=B^T(BAx-Cy)x^T +\lambda A,\\
 \frac{\displaystyle\partial F}{\displaystyle\partial B}(A,B,C,x,y)=(BAx-Cy)x^TA^T +\lambda B.
\end{array}\right.
\end{equation}
\hfill\qed

\subsection*{Proof of Lemma~\ref{lemma:aat-btb}}
 Multiplying the first equation in (\ref{eq:lingrad}) by $A^T$ on
   the right and subtracting the  second one multiplied by $B^T$ on the left
   immediately yields $B^TB-AA^T=0$ at a limit point since both
   derivatives $\bar{P}(A,B,C)$ and $\bar{Q}(A,B,C)$ are zero at such a
   point.
\hfill\qed

\subsection*{Proof of Lemma~\ref{lemma:dyn2}}
Let us first write the derivative of $E$ with respect to $A$. We have
\begin{equation}
    \displaystyle\frac{\partial E}{\partial A}=\displaystyle\frac{\partial F}{\partial
                                                  A}+Ayy^T-BAxy^T
= B^T(BAx-Ay)x^T-(BAx-Ay)y^T+\lambda A,
  \label{eq:Eflow}
\end{equation}
and thus
$\dot{A}=-(B^TR(A,B)-S(A,B)+\lambda A)$, where $S(A,B)=BA[xy^T]-A[yy^T]$.
Substituting in the temporal derivative of $1/2\Vert A\Vert _F^2$, we obtain
\begin{equation}
\begin{array}{lcl}
  \frac{d}{dt}[\frac{1}{2}\Vert A\Vert _F^2]&=&\frac{d}{dt}[\frac{1}{2}[\text{tr}(A^TA)]
                                        =-\text{tr}(A^T\dot{A})=\text{tr}(A^TB^TR(A,B)+A^TS(A,B)+
                                        \lambda A^TA)\\
                                    &=&-\text{tr}(\EE_{x,y}[x^TA^TB^TBAx+y^TA^TAy-x^TA^TB^TAy
                                        -y^TA^TBAx]+\lambda A^TA)\\
                                    &=&-(\EE_{x,y}[\Vert BAx-Ay\Vert ^2]+\lambda\Vert A\Vert _F^2)\\
                                    &\le& -\lambda \Vert A\Vert _F^2
\end{array}
\end{equation}
which implies the exponential convergence of $\Vert A\Vert _F^2$ toward zero.
\hfill\qed

\subsection*{Proof of Lemma~\ref{lemma:th1}}
Multiplying on both sides the first equality in Eq.~(\ref{eq:doteqs}) on the
right by $A^T$ and the second one on the left by $B^T$ and taking
the difference  yields
\begin{equation}
           B^T(\dot{B}+\lambda B)=(\dot{A}+\lambda A)A^T\\
 \end{equation}
Adding this equation to its transpose and
 multiplying both sides by $e^{2\lambda t}$ now yields
 \begin{equation}
   e^{2\lambda t}  (\dot{B}^TB+B^T\dot{B}+2\lambda B^TB) = e^{2\lambda t}(A\dot{A}^T+\dot{A}A^T+2\lambda AA^T),
\label{eq:btbder}
\end{equation}
from which we conclude that
\begin{equation}
  \frac{d}{dt}[e^{2\lambda t} B^TB]=\frac{d}{dt}[e^{2\lambda t}AA^T],
\label{eq:expder}
\end{equation}
and obtain $AA^T=B^TB+e^{-2\lambda t}K$, where $K$ is a constant independent of
the data. This implies in particular that $B^TB-AA^T\rightarrow 0$ as
$t\rightarrow+\infty$.
\hfill\qed

\subsection*{Proof of Proposition~\ref{prop:Eqchar}\label{sec:Eqchar}}
We only give there the proof for the SG algorithm since the proof for
the EMA algorithm follows the exact same reasoning with the extra
parameter $C$ known to be equal to $A$ at an equilibrium.

The equilibria $(A,B)$ of the SG algorithm are characterized by the
two equations
\begin{equation}
\left\{\begin{array}{l}
0=B^T(BA[xx^T]-A[yx^T])+\lambda A,\\
0=(BA[xx^T]-A[yx^T])A^T+\lambda B=BW-A[yx^T]A^T=0.\\
       \end{array}
     \right.
\label{eq:vareq}
\end{equation}
The fact that $B^TB=AA^T$ follows immediately from multiplying both
sides of the first condition by $A^T$ on the right and both sides of
the second one by $B^T$ on the left, then taking the difference.
Equation~(\ref{eq:Bchar}) follows immediately from the second
condition, the inverse being well defined when $\lambda>0$ since $W$
is symmetric positive definite in this case. Substituting this value
in the first equality of Eq.~(\ref{eq:vareq}) now yields
\begin{equation}
  W^{-1}A[xy^T]A^T\left(A[yx^T]A^TW^{-1}A[xx^T]-A[yx^T]\right)+\lambda A=0,
\end{equation}
and equilibria of the SG algorithm are exactly the pairs $(A,B)$ where
$A$ satisfies this condition and $B$ is given by
Eq.~(\ref{eq:Bchar}).  Assuming that $A$ has full rank $m$ and
multiplying both sides of this equation on the left by $W$ and on
the right by $A^T$ yields the equivalent condition
\begin{equation}
  A[xy^T]A^T\left(A[yx^T]A^TW^{-1}(W-\lambda\text{Id})-A[yx^T]A^T\right)
  +\lambda WAA^T=0.
\end{equation}
Multiplying both sides of this equation on the right by $W$, we
obtain the condition
\begin{equation}
  A(U^TU-V^TV)A^T=0,
\,\,\text{where}\,\,
  \left\{
    \begin{array}{l}
      U=A^TA[xx^T]+\lambda\text{Id},\\
      V=A[yx^T],
    \end{array}
  \right.
\end{equation}
which is equivalent to $U^TU=V^TV$ since we have assumed that $A$ has
full rank, and thus to Eq.~(\ref{eq:Achar}) as well.
\hfill\qed

Note that \cite{Tian21} give in Appendix D of their paper an
alternative characterization of the equilibria  in the case
  where $\lambda=0$, which essentially corresponds to the condition
of Eq.~(\ref{eq:Bchar}) in this case, {\em without} the characterization of $S=A^TA$
by Eq.~(\ref{eq:Achar}).

\subsection*{Proof of Corollary~\ref{corollary1}}
  The proof is textbook material and included for completeness. When
  $U$ is column orthogonal and $A=U\sqrt{S}_k$, we obviously have
  $A^TA=S_k$. Conversely, when $A^TA=S_k$, let us take
  $U=A\sqrt{S}_k^{-1}$ (the inverse is guaranteed to exist since $S_k$
  is positive definite). We have
  $U^TU=\sqrt{S}_k^{-1}S_k\sqrt{S}_k^{-1}=\text{Id}$.
\hfill\qed

\subsection*{Proof of Proposition~\ref{prop:genericity_full_rank}}
Fix $\epsilon>0$. There exist $0<\delta\leq \epsilon$ positive so that $T := [yx^{\top}] + \delta I$ and $ R:= [xx^{\top}] + \delta I$ are both invertible. To see this, it suffices to notice that $\delta \mapsto det( [yx^{\top}] + \delta I)$ is a non-zero polynomial, thus vanishes for a finite number of values $\delta$. Therefore, we can always find $\delta < \epsilon$ for which $det( [yx^{\top}] + \delta I) \neq 0$, so that $T$ is invertible. Furthemore, $R$ is necessarily invertible since it is the sum of the PSD matrix $[xx^{\top}]$ and the PD matrix $\delta I$.  The matrices $R$ and $T$ correspond to the covariances of the following perturbed variables:
\begin{align}
	\tilde{x} = x + \delta^{\frac{1}{2}}z,\qquad \tilde{y} = y + \delta^{\frac{1}{2}}z,
\end{align}
where $z$ is a standard isotropic gaussian. Thus we have $R = [\tilde{x}\tilde{x}^{\top}]$ and $T = [\tilde{y}\tilde{x}^{\top}]$. We will show that there exists a positive $\lambda_0$ small enough so that the following equation always admits a PD solution for any $ 0 \leq \lambda\leq \lambda_0$:
\begin{align}
	([\tilde{x}\tilde{x}^{\top}] S + \lambda  Id)( S [\tilde{x}\tilde{x}^{\top}] + \lambda  Id) = [\tilde{y}\tilde{x}^{\top}] S [\tilde{y}\tilde{x}^{\top}].
\end{align}
We will first establish existence of the solution for $\lambda = 0$, then show that the property still holds for $\lambda$ small enough.

\textbf{Case $\lambda =0$.} In this case, the equation simplifies to:
\begin{align}
	RS^2R=T^{\top}ST,
\end{align}
where we used the notation $R$ and $T$ for simplicity. Since, $R$ is invertible, we multiply both sides by $R^{-1}$ (left and right), which yields:
\begin{align}
	S^2 = (TR^{-1})^{\top}S(TR^{-1}).
\end{align}
Since the matrix $TR^{-1}$ is invertible, we can directly apply the technical Lemma \ref{prop:existance_2nd_order_eq}, stated below, which guarantees the existence of a PD solution $S^{\star}$ to the above equation.

\textbf{Case $\lambda>0$.} We will apply the implicit function theorem to show the existence of solutions for $\lambda$ small enough. Consider the matrix valued map $\mathcal{G}(\lambda, S)$ defined as:
\begin{align}
	\mathcal{G}(\lambda, S) = S^2 + \lambda(SR^{-1} + R^{-1}S) + \lambda^2 Id - (TR^{-1})^{\top}S(TR^{-1}).
\end{align}
We have already established that the equation $\mathcal{G}(0, S) =0 $ admits a solution $S^{\star}$. It suffices to prove that the partial differential $d_{S}\mathcal{G}(0, S^{\star})$ at $(0,S^{\star})$ is invertible. Since $d_{S}\mathcal{G}(0, S^{\star})$ is a linear map from the set of $m\times m$ matrices to itself, it suffices to establish its injectivity. Direct calculations show that $H\mapsto d_{S}\mathcal{G}(0, S^{\star})(H)$ is given by:
\begin{align}
	d_{S}\mathcal{G}(0, S^{\star})(H) = S^{\star}H+HS^{\star} - (TR^{-1})^{\top}H(TR^{-1}).
\end{align}
Using again the technical Lemma \ref{prop:existance_2nd_order_eq}, we know that the only solution to the equation $d_{S}\mathcal{G}(0, S^{\star})(H)=0$ is $H=0$. Hence, $d_{S}\mathcal{G}(0, S^{\star})$ is injective. Therefore, by the implicit function theorem, there exists a positive $\lambda_0$ so that for any $0\leq \lambda \leq \lambda_0$, the equation $\mathcal{G}(\lambda, S)=0$ admits a solution $S^{\star}$.

\begin{lemma}\label{prop:existance_2nd_order_eq}
Let $C$ be an invertible  $m\times m$ matrix. There exists a symmetric PD solution $X^{\star}$ to the following equation:
\begin{align*}
	X^2 = C^{\top}XC.
\end{align*}
Moreover, consider the following linear system $X^{\star}H + HX^{\star} - C^{\top}HC=0$ with unknown $H$.  The only solution to such system is $H=0$.
\end{lemma}
\begin{proof}
\textbf{Existence.} We will apply Brouwer's fixed point theorem to a suitable operator. Denote by $\mu$ and $\rho$ its smallest and largest eigenvalues of $C^{\top}C$ which are positive.  Define the following continuous map $\mathcal{G}$ over the set of $\mathbb{S}_{m}^{+}$ of symmetric PSD matrices of size $m\times m$:
	\begin{align*}
		\mathcal{G}(X) = (C^{\top}XC)^{\frac{1}{2}},
	\end{align*}
where the square root denotes the unique PSD square root of a PSD matrix. Note that $\mathcal{G}(X)$ is well defined for any $X\in \mathbb{S}_{m}^{+}$. Consider the following set of matrices:
\begin{align*}
	\mathbb{M} = \{X \in \mathbb{S}_{m}^{+}:  \quad  \mu I \leq X\leq \rho I\}.
\end{align*}
Then $\mathbb{M}$ is a convex compact subset of vector space of $n\times n$ matrices. We will show that $\mathcal{G}(\mathbb{M})\subset \mathbb{M}$, which will allow us to apply Brouwer fixed point theorem. Indeed, for any $X$ in $\mathbb{M}$, simple matrix inequalities yield:
\begin{align*}
	\mu C^{\top}C \leq C^{\top}XC\leq \rho C^{\top}C.
\end{align*}
Recalling that the symmetric square root preserves the matrix order, we directly get:
\begin{align*}
	\mu^{\frac{1}{2}} (C^{\top}C)^{\frac{1}{2}} \leq \mathcal{G}(X)\leq \rho^{\frac{1}{2}} (C^{\top}C)^{\frac{1}{2}}.
\end{align*}
However, by definition of $\rho$ and $\mu$, we have that $\mu I \leq C^{\top}C\leq \rho I$. Consequently, it follows that $\mu I \leq  \mathcal{G}(X)\leq \rho I$. Hence, we have  established that $\mathbb{M}$ is a stable set of the map $\mathcal{G}$. Since the map $\mathcal{G}(X)$ is continuous and $\mathbb{M}$ is a convex compact subset of the space of square matrices, it follows by Brouwer's fixed point theorem that there exists $X^{\star}$ satisfying the equation $\mathcal{G}(X^{\star})= X^{\star}$. After taking the square of such equation, we get that $X^{\star}$ is a solution to $X^2 = C^{\top}XC$.

\textbf{Uniqueness of the solution to the linear system.}  Let  $H$ be an $m\times m$ matrix solution to the linear system:
\begin{align*}
X^{\star}H + HX^{\star} - C^{\top}HC=0.
\end{align*}
We wish to show that $H=0$. We can multiply such equation (left and right) by $(X^{\star})^{-\frac{1}{2}}$ to get:
\begin{align*}
	(X^{\star})^{\frac{1}{2}}H(X^{\star})^{-\frac{1}{2}} + (X^{\star})^{-\frac{1}{2}}H(X^{\star})^{\frac{1}{2}} - (X^{\star})^{-\frac{1}{2}}C^{\top}HC(X^{\star})^{-\frac{1}{2}}=0
\end{align*}
Define $E = (X^{\star})^{-\frac{1}{2}}C^{\top}(X^{\star})^{\frac{1}{2}}$. By direct calculation and using the definition of $X^{\star}$ (i.e. the solution to the equation $X^2 = C^{\top}XC$), we have that $E$ satisfies $EE^{\top}= X^{\star}$. Now, consider the change of variables $\tilde{H} = (X^{\star})^{-\frac{1}{2}}H (X^{\star})^{-\frac{1}{2}}$. We can thus express the above  equation in terms of $\tilde{H}$ and $E$ and $X^{\star}$ as follows:
\begin{align*}
	X^{\star}\tilde{H}  + \tilde{H} X^{\star} - E\tilde{H} E^{\top}= 0.
\end{align*}
Since $E$ is invertible, we can further multiply the equation by $E^{-1}$ on the left and by $E^{-\top}$ on the right and use the identity $EE^{\top}= X^{\star}$ to get:
\begin{align*}
	E^{\top}\tilde{H}E^{-\top} + E^{-1}\tilde{H}E = \tilde{H}.
\end{align*}
The above equation directly implies that $ \tilde{H}$ must be symmetric. Furthermore, by direct calculation, we obtain the following expression for $\Vert \tilde{H}\Vert_F^2$:
\begin{align*}
	\Vert \tilde{H}\Vert_F^2  &= Tr\left((E^{\top}\tilde{H}E^{-\top} + E^{-1}\tilde{H}E)^{\top}(E^{\top}\tilde{H}E^{-\top} + E^{-1}\tilde{H}E)\right)\\
	&= \text{tr}(E^T\tilde{H}^2E^{-\top}) + 2\text{tr}(E^{\top}\tilde{H}E^{-\top} E^{-1}\tilde{H}E  ) + \text{tr}(E^{-1}\tilde{H}^2E)\\
	&= 2\Vert \tilde{H}^2 \Vert_F^2 + 2\Vert E^{-1}\tilde{H}E\Vert_F^2.
\end{align*}
The above identity can only be true if  $\Vert \tilde{H} \Vert_{F}^2 = 0$ which directly implies that $H=0$ since $X^{\star}$ is invertible.

\end{proof}

\subsection*{Proof of Proposition~\ref{prop:dynamics}\label{sec:dynprop}}
 We only present here the proof for the EMA procedure. The SG case is
  similar  and slightly simpler, and thus omitted for conciseness.
Let $(A,B,C)$ be an equilibrium point of our dynamical system and
$(A+D,B+E,C+F)$ a point in its neighborhood, where, like
$A$, $B$ and $C$, the matrices $D$, $E$ and $F$ are respectively
of size $n\times m$, $n\times n$ and $n\times m$.
To first order, we have
\begin{equation}
  \left\{
 \begin{array}{l}
   \dot{A}(A+D,B+E,C+F)
   \approx -B^T((E A+BD)[xx^T]-F[yx^T])-E^TR(A,B,C)
             -\lambda D,\\
   \dot{B}(A+D,B+E,C+F)
   \approx-((E
             A+BD)[xx^T]-F[yx^T])A^T-R(A,B,C)D^T-\lambda E,
   \\
   \dot{C}(A+D,B+E, C+F)=(1-\alpha)(D-F).
 \end{array} \right.
\end{equation}
The eigenvectors of the corresponding linear operator
in $D$, $E$ and $F$ and the corresponding
eigenvalues $\mu$ are thus characterized by
\begin{equation}
\left\{\begin{array}{l}
0=B^T((E A+BD)[xx^T]-F[yx^T])+E^TR(A,B,C)
             +(\lambda+\mu)D,\\
0=         ((E A+BD)[xx^T]-F[yx]^T)A^T+R(A,B,C)D^T+
         (\lambda+\mu)E, \\
         (1-\alpha)D=(1-\alpha+\mu)F.
       \end{array}
     \right.
     \label{eq:dyneq}
\end{equation}

Let us first consider the trivial equilibrium where $A=C=0$ and
$B=0$. Substituting these values in Eq.~(\ref{eq:dyneq}) shows that
for any triplet $(D,E,F)$ satisfying this equation we must
have either $\mu=-\lambda<0$ if $D$ or $E$ is nonzero, or
$\mu=\alpha-1$ if $D$ and $E$ are both zero. But, as noted before,
$\alpha$ is normally taken smaller than or equal to $1$ and we assume
here that $\alpha\neq 1$ (the moving average would not make much sense
otherwise since $\xi_t$ would be constant in that case), so $\mu<0$ in
that case as well. It follows that trivial equilibria are
asymptotically stable

Let us now consider the case of nontrivial equilibria where, in
particular, $A\neq 0$. Note that for any eigenvector triplet $(D,E,F)$
and associated eigenvalue $\mu$ satisfying Eq.~(\ref{eq:dyneq}),
either $1-\alpha+\mu$ is zero, in which case $F$ can take any value
and, as just observed, $\mu<0$, or it is not, in which case
$F=\beta D$ with $\beta=(1-\alpha)/(1-\alpha+\mu)$ so we can focus on
the first two equations.

Multiplying the first one on the right by $A^T$ and the second
one on the left by $B^T$, subtracting the two and using  Eq.~(\ref{eq:vareq}) yields
\begin{equation}
  (\lambda+\mu)M=-\lambda M^T\,\,\text{where}\,\, M=(D A^T-B^TE).
\end{equation}

Now, any matrix $U$ such that $U^T=aU$ for some scalar $a$ also
verifies $U=aU^T$ by taking the transpose on both sides, and thus
$U=aU^T=a(aU)=a^2U$, which means that, either $U=0$ or $U\neq 0$ and
$a^2=1$, with either $a=1$ and
$U$ symmetric or $a=-1$ and $U$ skew-symmetric.
In our case, when $M\neq 0$, it is either symmetric with
$\mu=-2\lambda$, or skew-symmetric with $\mu=0$.
In the first case, any equilibrium is asymptotically stable
according to   Theorem~\ref{th:Arnold} while, in the second one, all eigenvalues vanish
and nothing can be said to first order, which should not happen
generically.

Let us now prove that $\mu <0$  in the slightly more complicated
case $M=0$. Using $B^TE=D
A^T$, multiplying again the first equation in (\ref{eq:dyneq}) on the
right by $A^T$ (or the second one by $B^T$ on the left) and using
Eq.~(\ref{eq:vareq}) now yields
\begin{equation}
  \begin{array}{lcl}
  0&=&B^T((E A+BD)[xx^T]-F[yx^T])A^T+E^T R(A,B,C)A^T
  +(\lambda+\mu)D A^T\\
  &=&(D A^TA+AA^TD)[xx^T]-B^TF[yx^T])A^T-\lambda AD^T
  +(\lambda+\mu)D A^T,\\
    &=&N+P+\mu DA^T,
  \end{array}
  \label{eq:MM0}
\end{equation}
where
\begin{equation}
  \left\{
    \begin{array}{l}
      N=AA^TD[xx^T]A^T-\beta B^TD[yx^T]A^T-\lambda AD^T,\\
      P=DA^TA[xx^T]A^T+\lambda DA^T=DA^T(A[xx^T]A^T+\lambda\text{Id}).
    \end{array}
  \right.
  \label{eq:xxx}
\end{equation}
Here, $N$ is an $n\times n$ matrix of rank at most $m$ with $n>m$. It
is therefore singular with a kernel of dimension $n-m$. So is its
transpose.
Let us pick $u$ in $\text{Ker}(N^T)$ such that $AD^T u\neq
0$. Generically this is always possible since there is no reason for
$\text{Ker}(N^T)$ and $\text{Ker}(AD^T)$ to coincide.

Multiplying the second equation in~(\ref{eq:xxx}) on the left by $u^T$ and on the right
by $v=AD^Tu\neq 0$ now yields
\begin{equation}
  v^T(A[xx^T]A^T+\lambda\text{Id})v+\mu\Vert v\Vert ^2=0,
\end{equation}
and since $A[xx^T]A^T+\lambda\text{Id}$ is positive definite, we conclude
that $\mu<0$.
\hfill\qed

\subsection*{Proof of Proposition~\ref{prop:dynm1}}
  When $m=1$, $S=\Vert a\Vert ^2$ and Eq.~(\ref{eq:Achar}) can be rewritten as
\begin{equation}
(\rho\Vert a\Vert ^2+\lambda)^2=\tau^2\Vert a\Vert ^2,
\end{equation}
or equivalently
\begin{equation}
  \rho\Vert a\Vert ^2 -\varepsilon\tau\Vert a\Vert +\lambda=0\,\,\text{where}\,\,
  \varepsilon=\mp 1.
  \label{eq:quad0}
\end{equation}
A necessary and sufficient condition for real solutions of this
quadratic equation in $\Vert a\Vert $ to exist is that its discriminant $\Delta$
be nonnegative and they are nonnegative when $\varepsilon\tau$ is
itself nonnegative, i.e., $\varepsilon=\text{sign}(\tau)$, so
\begin{equation}
  \rho\Vert a\Vert ^2 -|\tau|\,\Vert a\Vert +\lambda=0.
  \label{eq:quad}
\end{equation}
These solutions indeed correspond to the two hyperspheres $S_1$
and $S_2$ defined above, which concludes the proof of the first part
of the proposition. These correspond exactly to the varieties
$\mathcal{A}_1$ and $\mathcal{A}_2$ associated with the two positive
roots $r_1^2$ and $r_2^2$ of the quadratic equation Eq.~(\ref{eq:Achar}) in $\Vert a\Vert ^2$ of course.
Now, let $a$ be an element of $S_i$ ($i=1,2$). According to
Eqs.~(\ref{eq:Bchar}) and~(\ref{eq:quad}),  we have
\begin{equation}
  B=\tau aa^TW^{-1}=\frac{\tau}{\rho\Vert a\Vert ^2+\lambda}aa^T=
  \frac{1}{r_i}\text{sign}(\tau) aa^T.
\end{equation}
This concludes the first part of the proof of the proposition.

Let us now turn to its second part, assume $\Delta\ge 0$ and consider
an equilibrium with $a$ in $S_i$.  As shown in the proof of
Prop.~\ref{prop:dynamics}, generically, all eigenvalues are negative
unless $M=da^T-B^TE= 0$. Substituting the value of $B$ in this
equation shows that $da^T=(\text{sign}(\tau)/r_i)aa^TE$, and since
eigenvectors are only defined up to scale we can pick $d=a$ and
$E^Ta=\text{sign}(\tau)r_ia$ (note that there exists an infinity of
$n\times n$ matrices $E$ verifying this equality, including
$E=B$). Substituting in Eq.~(\ref{eq:dyneq}) and using
Eqs.~(\ref{eq:quad0}) and~(\ref{eq:quad}) now yields
\begin{equation}
  0=(3\rho r_i^2-2\text{sign}(\tau)r_i+\lambda+\mu)a
  =(\mu-\lambda+\rho r_i^2)a=(\mu -2\lambda+|\tau|\,r_i)a,
\end{equation}
and thus, when $a\neq 0$, $\mu=2\lambda-|\tau|r_i$ (note that with
this value for $\mu$, $d=a$ and any matrix $E$ such that
$E^Ta=\text{sign}(\tau)r_ia$ satisfy Eq.~(\ref{eq:dyneq}) and are thus
indeed the $(d,E)$ part of the corresponding eigenvector). Now, a
$r_i=(|\tau|+\eta_i\sqrt\Delta)/2\rho$ where $\eta_1=-1$ and
$\eta_2=1$, so we have
\begin{equation}
  \mu=2\lambda-\frac{1}{2\rho}|\tau|(|\tau|+\eta_i|\tau|\sqrt\Delta)=\frac{-1}{2\rho}
  (\Delta+\eta_i|\tau|\sqrt\Delta)=\frac{-\Delta}{2\rho}(\sqrt\Delta+\eta_i|\tau|).
\end{equation}
In particular, a necessary and sufficient condition for $\mu$ to be positive is
that $\eta_1=-1$, corresponding to the hypersphere of
radius $r_1$, and that
\begin{equation}
  \tau^2>\Delta=\tau^2-4\rho\lambda,
\end{equation}
which is always true.
\hfill\qed

Proposition~\ref{prop:dynm1} appears to contradict
Proposition~\ref{prop:dynamics}. It does not since the case $m=1$ where
$A=a$ is a vector is non generic: in this case $\rho=[xx^T]$ is a
scalar (and thus commutes with all matrices involved) and
$B=(\varepsilon/\Vert a\Vert )aa^T$, and we can rewrite $N^T$
as
\begin{equation}
N^T =[(\rho - \frac{\varepsilon\tau}{\Vert a\Vert })(d\cdot a)a-\lambda
  d]a^T
  =-\lambda(\frac{1}{\Vert a\Vert ^2}aa^T+\text{Id})d a^T
    =-\lambda(\frac{1}{\Vert a\Vert ^2}aa^T+\text{Id})a a^T.
\end{equation}
In particular, since $((1/\Vert a\Vert ^2)aa^T +\text{Id})$ is positive
definite, $\text{Ker}(N^T)=\text{Ker}(a a^T)=\text{Ker}(a d^T)$, and
we cannot conclude that $\mu$ is negative in the last part of the
proof of Proposition~\ref{prop:dynamics}.  This is generically not the
case for $m>1$, so there is no contradiction.

\end{document}

%% file: acks.tex
\subsection*{Acknowledgments}
This work was supported in part by the French government under management of Agence Nationale de la Recherche as part of the “France 2030” program, PR[AI]RIE-PSAI projet, reference ANR23-IACL-0008. JP was supported in part by the Louis Vuitton/ENS chair in artificial intelligence, the Institute of Information \& Communications Technology Planning \& Evaluation (IITP) grant funded by the Korean Government (MSIT) (No. RS-2024-00457882, National AI Research Lab Project), and a Global Distinguished Professorship at the Courant Institute of Mathematical Sciences and the Center for Data Science at New York University.